\documentclass[letter,oneside]{article}
\usepackage[margin=1.0in]{geometry}

\usepackage{natbib}
\usepackage{algorithm}
\usepackage{algorithmic}
\usepackage{xcolor}

\usepackage{MixnMatch}

\title{Mix and Match: An Optimistic Tree-Search Approach for Learning
Models from Mixture Distributions}
\author{
Matthew Faw\thanks{University of Texas at Austin \dotfill
    \texttt{matthewfaw@utexas.edu}}
\and 
Rajat Sen\thanks{Amazon \dotfill
    \texttt{rajat.sen@utexas.edu}}
\and 
Karthikeyan Shanmugam\thanks{IBM Research NY \dotfill 
    \texttt{karthikeyan.shanmugam2@ibm.com}}
\and
Constantine Caramanis\thanks{University of Texas at Austin \dotfill 
    \texttt{constantine@utexas.edu}}
\and
Sanjay Shakkottai\thanks{University of Texas at Austin \dotfill 
    \texttt{sanjay.shakkottai@utexas.edu}}
}

\begin{document}

\maketitle

\begin{abstract}

We consider a covariate shift problem where one has access to several different
training datasets for the same learning problem and a small validation set
which possibly differs from all the individual training distributions. This
covariate shift is caused, in part, due to unobserved features in the datasets.
The objective, then, is to find the best mixture distribution over the training
datasets (with only observed features) such that training a learning algorithm
using this mixture has the best validation performance. Our proposed algorithm,
\match{}, combines stochastic gradient descent (SGD) with optimistic tree search and model re-use (evolving partially trained models with samples from different mixture distributions) over the space of mixtures, for this task. We prove simple regret guarantees for our algorithm with respect to recovering the optimal mixture, given a total budget of SGD evaluations. Finally, we validate our algorithm on two real-world datasets.

\end{abstract}
 
\newpage
\tableofcontents
\newpage

\section{Introduction}
\label{sec:intro}

The problem of {\em covariate shift} -- where the distribution on the covariates is different across the training and validation datasets -- has long been appreciated as an issue of central importance for real-world problems (e.g., \cite{shimodaira2000improving,gretton2009covariate} and references therein). Covariate shift is often ascribed to a changing population, bias in selection, or imperfect, noisy or missing measurements. Across these settings, a number of approaches to mitigate covariate shift attempt to re-weight the samples of the training set in order to match the target set distribution \cite{shimodaira2000improving,zadrozny2004learning,huang2007correcting,gretton2009covariate}. For example, \cite{huang2007correcting,gretton2009covariate} use unlabeled data to compute a good kernel reweighting. 

We consider a setting where covariate shift is due to {\em unobserved variables} in different populations (datasets). A motivating example is the setting of predictive health care in different regions of the world. Here, the unobserved variables may represent, for example, prevalence and expression of different conditions, genes, etc. in the makeup of the population. Another key example, and one for which we have real-world data (see Section \ref{sec:sims}), is predicting what insurance plan a customer will purchase, in a given state. The unobserved variables in this setting might include employment information (security at work), risk-level of driving, or state-specific factors such as weather or other driving-related features.
Motivated by such applications, we consider the setting where the joint distribution  (of observed, unobserved variables and labels) may differ across various populations, but the conditional distribution of the label ({\em conditioned on both observed and unobserved variables}) remains invariant. 
The goal, then, is to determine a mixture distribution over the input datasets (training populations) in order to optimize performance on the validation set.

The contributions in this paper are as follows:

\textbf{(i) Search based methods for covariate shift:} With latent/unobserved features, we  show in Section~\ref{sec:pformulation} that traditional methods such as moment matching cannot learn the best mixture distribution (over input datasets) that optimizes performance with respect to a validation set.
Instead, we show that searching over input mixture distributions using validation loss results in the recovery of the true model (with respect to the validation, Proposition~\ref{prop:validationLoss}). This motivates our tree search based approach. 

\textbf{(ii) \match{} -- Optimistic tree search over models:} We propose
\match{} -- an algorithm that is built on SGD and a variant of optimistic
tree-search (closely related to Monte Carlo Tree Search). Given a budget
(denoted as $\Lambda$) on the
total number of online SGD iterations, \match{} adaptively allocates this
budget to different population reweightings (mixture distributions over input
datasets) through an iterative tree-search procedure (Section~\ref{sec:algo}).
Importantly, \match{} expends a majority of the SGD iteration
budget on reweightings that are "close" to the optimal reweighting mixture by using two important ideas:\\
{\em (a) Parsimony in expending iterations:} For a reweighting distribution
that we have low confidence of being ``good,'' \match{} expends only a small number of SGD iterations to train the model; doing so, however, results in biased and noisy evaluation of this model, due to early stopping in training. \\
{\em (b) Re-use of models:} Rather than train a model from scratch, \match{} reuses and updates a partially trained model from past reweightings that are ``close'' to the currently chosen reweighting (effectively re-using SGD iterations from the past).

\textbf{(iii) SGD concentrations without global gradient bounds:} The analysis
of \match{} requires a new concentration bound on the error of the final iterate
of SGD. 
{\color{black}
Instead of assuming a uniform bound on the norm of the stochastic
gradient over the domain, as is typical in the stochastic optimization
literature, we directly exploit
properties of the {\em averaged} loss (strong convexity) and individual loss
(smoothness) combined with a bound on the norm of the stochastic gradient
\textit{at a single point} to bound the norm of the stochastic gradient at each
iterate. 
Using a single parameter ($\Lambda$, the budget allocated to \match{}),
we are able to balance the worst-case 
growth of the norm of the
stochastic gradient with the probability of failure of the SGD concentration.
}
This new result
(Theorem~\ref{thm:SGD_improved}) provides tighter high-probability guarantees
on the error of the final SGD iterate in settings where the diameter of the
domain is large and/or cannot be controlled. 
 \section{Related Work}
\label{sec:rwork}
Transfer learning has assumed an increasingly important role, especially in settings where we are either computationally limited or data-limited but can leverage significant computational and data resources on domains that differ slightly from the target domain \cite{raina2007self,pan2009survey,dai2009eigentransfer}. This has become an important paradigm in neural networks and other areas \cite{bengio2011deep,yosinski2014transferable,oquab2014learning,kornblith2018better}. 
A related problem is that of covariate shift \cite{shimodaira2000improving,zadrozny2004learning,gretton2009covariate}, where the target population distribution may differ from that of the training distribution. 
Some recent works have considered addressing this problem by reweighting samples from the training dataset so that the distribution better matches the test set, for example by using unlabelled data \cite{huang2007correcting, gretton2009covariate} or variants of importance sampling \cite{sugiyama2007covariate,sugiyama2008direct}.
The authors in~\cite{mohri2019agnostic} study a related problem of learning from different datasets, but provide minimax bounds in terms of an agnostically chosen test distribution.

Our work is related to, but differs from all the above.  As we explain in Section \ref{sec:psetting}, we share the goal of transfer learning: we have access to enough data for training from a family of distributions that are different than the validation distribution (from which we have only enough data to validate). However, to address the effects of  latent features, we adopt an optimistic tree search approach -- something that, as far as we know, has not been undertaken.

A key component of our tree-search based approach to the covariate shift
problem is the computational budget.
We use a single SGD iteration as the currency denomination of our budget, which requires us to minimize the number of SGD steps in total that our algorithm computes, and thus to understand the final-iterate optimization error of SGD in high probability. There are many works deriving error bounds on the final SGD iterate in expectation (e.g. \cite{bubeck15, bottou2018optimization, sgdHogwild}) and in high probability (e.g. \cite{rakhlin2012making,harvey2018tight} and references therein). However, to our knowledge, optimization error bounds on the final iterate of SGD when the stochastic gradient is assumed bounded only at the optimal solution exist only in expectation \cite{sgdHogwild}. We prove a similar result in high probability.

Optimistic tree search makes up the final important ingredient in our algorithm. These ideas have been used in a number of settings \cite{bubeck2011x,grill2015black}. Most relevant to us is a recent extension of these ideas to a setting with biased search \cite{sen2018multi,sen2019noisy}. 
 \section{Problem Setting and Model}
\label{sec:psetting}
 
\subsection{Data model}
Each dataset $\DD$ consists of samples of the form 
$z=(\xx,y)\in\RR^{d}\times \RR$, where $\xx$ corresponds to the \textit{observed} 
feature vector, and $y$ is the corresponding label.
Traditionally, we would regard dataset $\DD$ as governed by a distribution 
$p(\xx,y)$. However, we consider the setting where each sample $z$ is a 
\textit{projection} from some higher dimensional vector $\hat{z}=(\xx,\uu,y)$,
where $\uu\in\RR^{\hat{d}}$ is the unobserved feature vector. The corresponding
distribution function describing the dataset is thus $p(\xx,\uu,y)$.
This viewpoint allows us to model, for example, predictive healthcare
applications where the unobserved features $\uu$ could represent uncollected,
region specific information that is potentially useful in the prediction task 
(e.g., dietary preferences, workday length, etc.).

We assume access to \textbf{$K$ training datasets} $\{\DD_i\}_{i=1}^K$
(e.g., data for a predictive healthcare task collected in $K$ 
different countries)
with corresponding p.d.f.'s $\{p_i(\xx,\uu,y)\}_{i=1}^K$
through a \textit{sample oracle} to be described shortly.
Taking $\AA := \{\aa\in\RR^K : \aa \succeq \0,\|\aa\|_1 = 1\}$ as the 
$(K-1)$-dimensional mixture simplex, for any $\aa\in\AA$, we denote the
mixture distribution over the training datasets as 
$\pa(\xx,\uu,y):=\sum_{i=1}^K \aa_i p_i(\xx,\uu,y)$. 
Samples from these datasets may be obtained through a \textit{sample oracle} which,
given a mixture $\aa$, returns an independent sample from the corresponding mixture
distribution.
In the healthcare example, sampling from $\pa$ would mean first sampling an 
index $i$ from the multinomial distribution represented by $\aa$ 
and then drawing a new medical record from the database of the $i$th country.
Additionally, we have access to a \textit{small} (see Remark \ref{rem:smallValidationRegime}) validation dataset $\DD_{te}$ with corresponding 
distribution $\pte(\xx,\uu,y)$, for example, from a new country where only
limited data 
has been collected. We are interested in training a predictive model for the validation distribution, but we \textit{do not} have oracle sampling access 
to this distribution -- if we did, we could simply train a model through SGD directly on this
dataset. Instead, we only assume \textit{oracle access to evaluating the 
validation loss of a constrained set of models} (we define our loss model and the constrained set shortly).
We make the following assumptions on the validation distribution:

\begin{assumption}
\label{assump:validationInConvHull}
We assume that the marginal distribution of observed and
unobserved features for the validation 
distribution lies in the convex hull
of the corresponding training distributions -- that is, there exists some $\aa^* \in \AA$ such that 
$\pte(\xx,\uu)=\pas(\xx,\uu)$.
\end{assumption}

Additionally, we make the following assumption on the conditional distribution of
the validation labels, which we refer to as \textbf{conditional shift invariance}:
\begin{assumption}[Conditional Shift Invariance]
\label{assump:condShiftInvariance}
We assume that the distribution of labels \textit{conditioned on both observed
and unobserved features} is fixed for all training and validation distributions. That 
is, for each $i\in[K]$
\begin{align*}
    p_i(y | \xx,\uu) = p(y | \xx, \uu) = \pte(y | \xx, \uu)
\end{align*}
for some fixed distribution $p$.
\end{assumption}
We note that Assumption \ref{assump:condShiftInvariance} generalizes the covariate
shift assumption of \cite{heckman1977sample,shimodaira2000improving} ($p(y | x)$ is fixed) to account for \textit{latent variables}.
 
\subsection{Loss function model} 
We denote the loss for a particular sample
$z$ and model $\ww\in \WW :=\RR^m$ as $f(\ww; z).$ For any mixture distribution
$\aa\in\AA$, we denote $\Fa(\ww) := \EE_{z\sim \pa}[f(\ww; z)]$ as the
\textit{averaged} loss function over distribution $\pa$.
Note that when $\aa$ is clear from context, we write $F(\ww)$.
Similarly, we
denote $\Fte(\ww) := \EE_{z\sim \pte}[f(\ww; z)]$ as the averaged validation loss.
We place the following assumptions on our loss function, similar 
to~\cite{sgdHogwild} (refer to Appendix \ref{sec:defs} for these standard
definitions):
\begin{assumption}
\label{assump1}
For each
loss function $f(\cdot; z)$ corresponding to a \textbf{sample}
$z\in\mathcal{Z}$,
we assume that $f(\cdot; z)$ is: \emph{(i)}~\textbf{$\beta$-smooth}
and \emph{(ii)}~\textbf{convex}.

Additionally, we assume that, for each $\aa\in\AA$, the \textbf{averaged}
loss function
$\Fa(\cdot)$ is: \emph{(i)}~\textbf{$\mu$-strongly convex}
and \emph{(ii)}~\textbf{$L$-Lipschitz}.
\end{assumption}

Notice that Assumption \ref{assump1} requires only the \textit{averaged} loss function $\Fa(\cdot)$
-- \textit{not} each individual loss
function $f(\cdot; z)$ -- to be strongly convex.
We additionally assume the following bound on the gradient of $f$ at $\bb^*$ along every sample path:

\begin{assumption}[A \textit{weaker} gradient norm bound]\label{assumGrad}
For all $\aa\in\AA$, there exists constants $\GG(\aa)\in\RR_+$ such that $\| \nabla f(\bb^*(\aa); z) \|_2^2 \leq \GG{(\aa)}.$ { When $\aa$ is clear from context, we write $\GG.$}
\end{assumption}

We note that Assumption \ref{assumGrad} is \textit{weaker} than the typical universal bound on the norm of the stochastic gradient assumed in, for example, \cite{rakhlin2012making, harvey2018tight}, and is taken from
\cite{sgdHogwild}.
  \section{Problem Formulation}
\label{sec:pformulation}

Given $K$ training datasets 
$\{\DD_i\}_{i=1}^K$ (e.g., healthcare data from $K$ countries) and a small (see Remark \ref{rem:smallValidationRegime}),
labeled validation dataset 
$\DD_{te}$ (e.g., preliminary data collected in a new country), we wish to find a model $\hat{\ww}$ such that the loss averaged
over the validation \textit{distribution}, $\pte$, is as small
as possible, using a computational budget to be described shortly. Under the notation introduced in Section 
\ref{sec:psetting}, we wish to approximately solve the optimization problem:
\begin{align}
\label{eq:goalW}
    \min_{\ww\in\WW} \Fte(\ww),
\end{align}
subject to a \textbf{computational budget of $\Lambda$ SGD iterations}.
A computational budget is often used in online optimization
as a model for constraints on the number of
i.i.d.\ samples available to the algorithm (see for example the introduction to Chapter 6 in
\cite{bubeck15}).

Note that one \textit{could} run SGD directly
on the validation dataset, $\Dte$, in order to minimize the expected loss on 
this population,
as long as the number of SGD steps is \textit{linear} 
in the size of $\Dte$ \cite{hardt2015train}.
When the number of validation samples is \textit{small relative to the
computational budget $\Lambda$}, such as in the
predictive healthcare
example where little data from the new target country is available,
the resulting error guarantees
of such a procedure will be correspondingly weak.
Thus, we hope to leverage
both training data \textit{and} validation data in order to solve (\ref{eq:goalW}).

Though we cannot train a model using $\Dte$,
we will assume $\Dte$ is sufficiently large to
obtain an accurate estimate of the validation loss. We model evaluations of
validation loss through oracle access to $\Fte(\cdot)$, which may be queried \textbf{only on models
trained by running at least one SGD iteration on some mixture distribution over the training datasets.}

\begin{remark}[Small validation dataset regime]
\label{rem:smallValidationRegime}
Under no assumptions on the usage of the (size $n$) validation set,
only $k~=~O(n^2)$ queries can be made while maintaining nontrivial generalization
guarantees \cite{bassily2016algorithmic, mania2019model}. When tracking only
the best model, as in \cite{blum2015ladder,hardt2017climbing}, $k$ can be
roughly \textbf{exponential} in the size of the validation set. While our
setting is more similar to this latter setting, a precise characterization of
the sample complexity, and thus of the precise bounds on the size of the
validation set, is important. Here we focus on the computational aspects, and
leave the formalization of generalization guarantees in our setting to future
work.
\end{remark}

Let $\ww^*(\aa) := \argmin_{\ww\in\WW} \Fa(\ww)$ be the optimal model for
training mixture distribution $\pa$. Similarly, let us denote $\hat{\ww}(\aa)$
as the model obtained after running $1\leq T \leq \Lambda$ steps of online SGD on $\pa$.
Then we can \textit{minimize validation 
loss}
$\Fte(\cdot)$ by \textbf{(i)}~\textit{iteratively selecting mixtures} $\aa\in\AA$,
\textbf{(ii)}~using a
\textit{portion} of the SGD budget to \textit{solve for} $\hat{\ww}(\aa)$, and 
\textbf{(iii)}~\textit{evaluating the
quality of the selected mixture} by obtaining the validation loss
$\Fte(\hat{\ww}(\aa))$ (through oracle access, as discussed earlier).
That is, using $\Lambda$ total SGD iterations, we can find a \textit{mixture
distribution} $\aa(\Lambda)$ and \textit{model} $\hat{\ww}(\aa(\Lambda))$ so that
$\Fte(\hat{\ww}(\Lambda))$ is as close as possible to
\begin{align}
\label{eq:goalAlpha}
     \min_{\aa\in\AA} G(\aa) = \min_{\aa\in\AA} \Fte(\ww^*(\aa)),
\end{align}
where $G(\aa) := \Fte(\ww^*(\aa))$ is the test loss evaluated at the
\textit{optimal model for $\pa$}. 

Under our Assumptions \ref{assump:validationInConvHull}
and \ref{assump:condShiftInvariance}, we have the following 
connection between training and validation loss, which establishes 
that solving (\ref{eq:goalW}) and (\ref{eq:goalAlpha}) are \textbf{equivalent}:

\begin{proposition}
\label{prop:validationLoss}
By Assumptions \ref{assump:validationInConvHull} and \ref{assump:condShiftInvariance}, the validation loss can be written in terms of mixtures of training loss:
\begin{align}
    \Fte(\ww) = \Fas(\ww)
\end{align}
for each $\ww\in\WW$, where $\aa^*$ is the mixture specified by Assumption
\ref{assump:validationInConvHull}.
As a consequence, finding $\ww^*$ which solves (\ref{eq:goalW}) is
\textbf{equivalent} to finding the mixture $\aa^*$ and corresponding model
$\ww^*(\aa^*)$ which solves (\ref{eq:goalAlpha}), since
$\Fte(\ww^*) = \Fas(\ww^*) = \Fas(\ww^*(\aa^*))$.
\end{proposition}

Proposition \ref{prop:validationLoss} follows immediately by noting that, from Assumptions \ref{assump:validationInConvHull} and \ref{assump:condShiftInvariance},
$\pte(\xx,\uu,y) = \pas(\xx,\uu,y)$, 
and thus $\pte(\xx,y)=\pas(\xx,y)$, and using the definition of $\Fte$.

We take as our objective to minimize simple regret with
respect to the optimal model $\ww^*(\aa^*)$:
\begin{align}
    \label{eq:simpleRegret}
    R(\Lambda) := G(\aa(\Lambda)) - \min_{\aa\in\AA} G(\aa).
\end{align}
That is, we measure the performance of our algorithm by the difference in validation
loss between \textit{the best model corresponding to our final selected mixture, 
$\ww^*(\aa(\Lambda))$} and the best model for the validation loss, $\ww^*(\aa^*)$.

\begin{remark}[\textbf{Difficulties with moment matching and domain invariant representations}] Note that we cannot learn $\aa^*$ simply by matching the mixture distribution over the  training sets to that of the validation set (both with only the observed features and labels). This is because $p_k(x,u)$ decomposes as $p_k(x,u) = p_k(x) p_k(u|x),$ where $p_k(u|x)$ is unknown and potentially differs across datasets. Thus, in a setting with unobservable features, approaches that try to directly learn the mixture weights by comparing with the validation set (e.g., using an MMD distance or moment matching) learns the wrong mixture weights. Further, our scenario also admits cases where the observed $p(y|x)$ (label distribution conditioned on observed variables) can shift which is non-trivial. In fact, when observed conditional distribution of labels differ between training and validation, strong lower bounds exist on many variants of another popular method called domain invariant representation (see Corollary $4.1$ in \cite{zhao2019learning}).
\end{remark}
 \section{Algorithm}
\label{sec:algo}

We now present \match{} (Algorithm \ref{algo:mixnmatch}), our proposed algorithm for minimizing 
$G(\aa) = \Fte(\ww^*(\aa))$ over the mixture simplex $\AA$ using $\Lambda$ total
SGD iterations.
To solve this minimization problem, our algorithm must search over the mixture simplex,
and for each $\aa\in\AA$ selected by the algorithm, \textit{approximately} evaluate
$G(\aa)$ by obtaining an approximate minimizer $\hat{\ww}(\aa)$ of $\Fa(\cdot)$
and evaluating $\widehat{G}(\aa) = \Fa(\hat{\ww}(\aa))$.
Two main ideas underlie our algorithm: \textbf{parsimony in expending
SGD iterations} -- using a small number of iterations for mixture distributions that we have a low
confidence are ``good'' -- and \textbf{model reuse} -- using models trained on nearby
mixtures as a starting point for training a model on a new mixture distribution. We now
outline why and how the algorithm utilizes these two ideas. In Section \ref{sec:results},
we formalize these ideas.

\textbf{Warming up: model search with optimal mixture.}
By Proposition \ref{prop:validationLoss}, 
$G(\aa) = \Fas(\ww^*(\aa))$ for all $\aa\in\AA.$ Therefore,
if we were given $\aa^*$ a priori, then we could run stochastic gradient descent
to minimize the loss over this mixture distribution on the training datasets,
$\Fas(\cdot)$, in order to find an $\varepsilon-$approximate solution to $\ww^*(\aa^*)$,
the desired optimal model for the validation distribution. In our experiments (Sections \ref{sec:sims} and Appendix \ref{sec:expers-appendix}), we
will refer to this algorithm as the \texttt{Genie}.
\textit{Our algorithm, thus, will be tasked to find a mixture close to $\aa^*$.}

\textbf{Close mixtures imply close optimal models.}
Now, suppose that instead of being given $\aa^*$,    
we were given some other $\hat{\aa}\in\AA$ which is close to $\aa^*$ in $\ell_1$ distance.
Then, as we will prove in Corollary~\ref{cor:alphaGeometricDecay}, 
we know that the optimal model for this alternate distribution
$\ww^*(\hat{\aa})$ is close to $\ww^*(\aa^*)$ in $\ell_2$ distance, and additionally,
$G(\hat{\aa})$ is close to $G(\aa^*)$. In fact, this property is not special to mixtures
close to $\aa^*$, but holds more generally for any two mixtures that are close in
$\ell_1$ distance.
\textit{Thus, our algorithm needs only to find a mixture $\hat{\aa}$ sufficiently close
to $\aa^*$.}

\textbf{Smoothness of $G(\cdot)$ and existence of ``good'' simplex partitioning 
implies applicability of optimistic tree search algorithms.}
This notion of smoothness of $G(\aa)$ immediately implies that we can use the optimistic 
tree search framework similar to \cite{bubeck2011x,grill2015black} in order to minimize
$G(\aa)$ by performing a tree search procedure over hierarchical partitions of the mixture 
simplex $\AA$ -- indeed, in this literature, such smoothness conditions are directly assumed. 
Additionally, the existence of a hierarchical partitioning such that the diameter of 
each partition cell decays exponentially with tree height is also assumed. In
our work, however, we prove in Corollary~\ref{cor:alphaGeometricDecay} that the smoothness
condition on $G(\cdot)$ holds, and by using the simplex bisection strategy described in
\cite{simplexBisection}, the cell diameter decay condition also holds.
\textit{Thus, it is natural to design our algorithm in the tree search framework.}

\textbf{Tree search framework.}
\match{} proceeds by constructing a binary partitioning tree $\TT$ over the space of mixtures $\AA$.
Each node $(h,i)\in\TT$ is indexed by the height (i.e. distance from the root node) $h$ and 
the node's index $i \in [2^h]$ in the layer of nodes at height $h$.
The set of nodes $V_h = \{(h,i) : i \in [2^h]\}$ at height $h$ are associated with a 
partition $\mathcal{P}_h = \{\mathcal{P}_{h,i} : i\in[2^h]\}$ of the mixture simplex 
$\triangle$ into $2^h$ disjoint partition cells whose union is $\AA$.
The root node $(0,1)$ is associated with the entire simplex $\AA$, and two children of 
node $(h,i)$, $\{(h+1, 2i-1), (h+1, 2i)\}$ correspond to the two partition cells of the parent
node's partition. The resulting hierarchical partitioning will be denoted $\mathcal{P} = \cup_h 
\mathcal{P}_h$,
and can be implemented using the simplex bisection strategy of \cite{simplexBisection}.
\textit{Combined with the smoothness results on our objective function, $\TT$ gives
a natural structure to search for $\aa^*$.}

\textbf{Multi-fidelity evaluations of $G(\cdot)$ -- associating $\TT$ with
mixtures and models.}
We note that, in our setting, $G(\aa)=\Fte(\ww^*(\aa))$ cannot be directly evaluated, since we cannot
obtain $\ww^*(\aa)$ explicitly, but only an approximate minimizer $\hat{\ww}(\aa)$.
Thus, we take inspiration from recent works in \textit{multi-fidelity} tree-search 
\cite{sen2018multi,sen2019noisy}.
Specifically, using a height-dependent SGD budget function
$\lambda(h)$, the algorithm takes $\lambda(h; \delta)$ SGD steps using some selected mixture 
$\aa_{h,i}\in\mathcal{P}_{h,i}$ to obtain 
an approximate minimizer $\hat{\ww}(\aa_{h,i})$ and evaluates the validation loss
$\Fte(\hat{\ww}(\aa_{h,i}))$ to obtain an estimate for $G(\aa_{h,i})$. \textit{$\lambda(\cdot)$ is 
designed so that estimates
of $G(\cdot)$ are ``crude'' early during the tree-search procedure and more refined deeper in 
the search tree.}

\textbf{Warm starting with the parent model.}
When our algorithm, \match{} selects node $(h,i)$, it creates child nodes $\{(h+1,2i-1), (h+1,2i)\}$,
and runs SGD on the associated mixtures $\aa_{h+1,2i-1}$ and $\aa_{h+1,2i}$,
starting each SGD run with initial model $\hat{\ww}(\aa_{h,i})$, \textit{the final iterate
of the parent node's SGD run.}
Since $\aa_{h,i}$ and $\aa_{h+1,j}$ ($j\in\{2i-1,2i\}$) are exponentially close
as a function of $h$ (as a consequence of our simplex partitioning strategy),
so too are $\ww^*(\aa_{h,i})$ and $\ww^*(\aa_{h+1,j})$ (since close mixtures implies 
close models). Thus, as long as the parent's final iterate is exponentially
close to $\ww^*(\aa_{h,i})$, then the initial iterate for the SGD runs associated
to the child nodes will also be exponentially close to their associated solution,
$\ww^*(\aa_{h+1,j})$. 
\textit{This implies that that a good initial condition of 
weights for a child node's model is that resulting from the final iterate of the parent's model.}

\textbf{Constant SGD steps suffice for exponential error improvement.}
In a noiseless setting (e.g., the setting of Theorem 3.12 in \cite{bubeck15}), optimization error
scales linearly in the squared distance between the initial model and the optimal model, and thus, 
in this setting, we could simply take a constant number of gradient descent steps to obtain a model
with error exponential in $h+1$. However, in SGD, optimization error depends not only
on the initial distance to the optimal model, but also on the noise of the stochastic gradient.
Under our $\beta$-smoothness assumption and Assumption \ref{assumGrad}, however, we can show
that, until we hit the noise floor of $\GG(\aa)$ (the bound on the norm of the gradient \textit{only}
at the optimal model $\ww^*(\aa)$), the noise of the stochastic gradient \textit{also decays 
exponentially with tree height} (see e.g. Lemma \ref{lem:centeredLem2} in the Appendix for a proof).
\textit{As a consequence, until we hit this noise floor, we may take a constant number of SGD
steps to exponentially improve the optimization error as we descend our search tree. In fact, all of
our experiments (Section \ref{sec:sims} and Appendix \ref{sec:expers-appendix}) use a height-independent 
budget function $\lambda$.}

{\bf Growing the search tree.} \match{} constructs the search tree in the same manner as MFDOO from \cite{sen2018multi}.
Initially, $\TT=\{(0,1)\}$, and until the SGD budget $\Lambda$ has been exhausted,
the algorithm proceeds by selecting the node $(h,i)$ from the set of leaf nodes
that has the smallest estimate (denoted $b_{h,i}$) 
for $G(\cdot)$ \textit{for any mixture} 
within the leaf's corresponding partition cell, $\mathcal{P}_{h,i}$.
\textit{In this manner, we can expect to obtain similar simple regret guarantees
as those obtained for MFDOO.}

\begin{algorithm}[tb] 
    \small
    \caption{\match{}: Tree-Search over the mixtures of training datasets }          
    \label{algo:mixnmatch}                           
    \begin{algorithmic}[1]
        \REQUIRE{Real numbers $\nu_2>0$, $\rho_2 \in (0,1)$ as specified in 
            Corollary~\ref{cor:alphaGeometricDecay}, hierarchical partition 
            $\mathcal{P}$ of $\AA$, SGD budget $\Lambda\geq 0$.}
        \STATE Expand the root node using Algorithm~\ref{algo:query} 
            and form two leaf nodes $\mathcal{T}_t = \{(1,1),(1,2)\}$. 
        \STATE  Cost (Number of SGD steps used): $C = 2\lambda(0)$
        \WHILE{$C \leq \Lambda$}
            \STATE Select the leaf $(h,j) \in \mathrm{Leaves}(\TT_t)$ with 
                minimum $b_{h,j} := F^{(te)}(\hat{\ww}(\aa_{h,j})) - 2\nu_2\rho_2^h$. 
            \STATE Add to $\TT_t$ the $2$ children of $(h,j)$ (as determined by
                $\mathcal{P})$ by querying them using Algorithm~\ref{algo:query}.
            \STATE $C = C + 2\lambda(h+1)$. 
        \ENDWHILE
        \STATE Let $h(\Lambda)$ be the height of $\TT_t$
        \STATE  Let $ i^* := \argmin_{{i}} F^{(te)}(\hat{\ww}(\aa_{h(\Lambda), i}))$. 
        \STATE Return $\aa_{h(\Lambda), i^*}$ and $\hat{\bb}(\aa_{h(\Lambda), i^*})$. 
    \end{algorithmic}
\end{algorithm}

\begin{algorithm}[tb]   
\small
    \caption{\eval{}: Optimize over the current mixture and evaluate}          
    \label{algo:query}                           
    \begin{algorithmic}[1]
        \REQUIRE{Parent node $(h,i)$ with model $\hat{\bb}(\aa_{h,i})$,
            $\nu_2 > 0$, $\rho_2~\in~(0,1)$}
        \STATE{// Iterate over new child node indices}
        \FOR {$(h', i') \in \{(h+1, 2i-1), (h+1, 2i)\}$} 
\STATE Let $\aa := \aa_{h',i'}\in\mathcal{P}_{h',i'}$ and $\bb_0 := \hat{\bb}(\aa_{h,i})$.
            \FOR {$t = 1,...,T := \lambda(h)$ (see Corollary~\ref{cor:step})} 
                \STATE $\bb_t = \bb_{t-1} - \eta_t \nabla f(\bb_{t-1};z_t)$ for $z_t \sim p^{(\aa)}$. 
            \ENDFOR
            \STATE Obtain test error $F^{(te)}(\bb_T)$
            \STATE Set node estimate: $b_{h',i'} = F^{(te)}(\bb_T) - 2\nu_2\rho_2^{h'}.$
            \STATE Set final model: $\hat{\ww}(\aa_{h',i'}) = \ww_T$.
        \ENDFOR
    \end{algorithmic}
\end{algorithm}
 \section{Theoretical Results}
\label{sec:results}

We now present the theoretical results which formalize the intuition outlined
in Section \ref{sec:algo}. All proofs can be found in the Appendix.
 
\subsection{Close mixtures imply close solutions}
Our first result shows that the optimal weights with respect to the two distributions $p^{(\aa_1)}$ and $p^{(\aa_2)}$ are close, if the mixture weights $\aa_1$ and $\aa_2$ are close. This is the crucial observation upon which Corollary \ref{cor:alphaGeometricDecay} relies.

\begin{theorem}
\label{thm:claim1}
Consider a loss function $f(\bb; z)$ which satisfies Assumptions \ref{assump1} and \ref{assumGrad},
and a convex body $\mathcal{X} = \mathrm{Conv}\{\bb^*(\aa) \in \mathcal{W} \mid \aa \in \mathcal{A}\}$.
Then for any $\aa_1,\aa_2\in\AA$,
$\norm{\bb^*(\aa_1) - \bb^*(\aa_2)}_2  \leq \frac{2\sigma \norm{\aa_1 - \aa_2}_1}{\mu}. $
where $\sigma^2 = \sup_{\bb, \bb'\in\mathcal{X}}\sup_{\aa\in\mathcal{A}} \beta^2 \|\bb - \bb'\|^2 + \GG(\aa).$

\end{theorem}

The above theorem is essentially a generalization of Theorem 3.9 in~\cite{hardt2015train} to the case when only $\EE[f],$ not $f$, is strongly convex. Theorem~\ref{thm:claim1} implies that, if the partitions are such that for any cell $(h,i)$ at height $h$, $\norm{\aa_1 - \aa_2}_1 \leq \nu'\rho^{h}$ for all $\aa_1,\aa_2 \in (h,i)$, where $\rho \in (0,1)$, then we have that $\norm{\bb^*(\aa_1) - \bb^*(\aa_2)}_2 \leq \nu_1\rho^h$, for some $\nu_1 \geq 0$. We note that such a partitioning  does indeed exist:

\begin{corollary}[of Theorem \ref{thm:claim1}]\label{cor:alphaGeometricDecay}
There exists a hierarchical partitioning $\mathcal{P}$ of the simplex of mixture weights $\mathcal{A}$ 
(namely, the simplex bisection strategy described in \cite{simplexBisection})
such that, for any cell $(h,i)\in\mathcal{P},$ and any $\aa_1, \aa_2\in (h,i),$
\begin{align}
    \|\aa_1~-~\aa_2\|_1~\leq~\sqrt{2 K}\left(\frac{\sqrt{3}}{2}\right)^{\frac{h}{K-1} - 1},
\end{align}
where $K-1 = \mathrm{dim}(\AA)$. 
Combined with Theorem \ref{thm:claim1}, this implies
\begin{align}
    \| \ww^*(\aa_1) - \ww^*(\aa_2) \|_2^2 \leq \nu_1 \rho^h 
\end{align}
and
\begin{align}
    | G(\aa_1) - G(\aa_2) | \leq \nu_2 \rho_2^h,
\end{align}
where
$\nu_1 = \left(\frac{4\sigma\sqrt{2K}}{\sqrt{3}\mu}\right)^2$, 
$\rho = \left(\frac{\sqrt{3}}{2}\right)^{\frac{2}{K-1}}$,
$\nu_2 = L \sqrt{\nu_1}$, and
$\rho_2 = \sqrt{\rho}.$
\end{corollary}
Refer to Appendix \ref{sec:claim1} for the proofs of these claims.
 
\subsection{High probability SGD bounds without the uniform gradient bound yields a budget allocation strategy}
We now show how to allocate our SGD budget as we explore new nodes
in the search tree.
To begin, let us consider how an approximately optimal
model $\hat{\ww}(\aa_{h,i})$ associated with some node $(h,i)\in\TT$
could be used to find $\hat{\ww}(\aa_{h',i'})$ a child node $(h',i')$.
By Corollary~\ref{cor:alphaGeometricDecay}, $\aa_{h,i}$ and $\aa_{h',i'}$
are exponentially (in $h$) close in
$\ell_1$ distance, so $\ww^*(\aa)$ and $\ww^*(\aa')$ are correspondingly close
in $\ell_2$ distance.
This leads us to hope that, if we were to obtain a good enough estimate to the 
problem at the parent node and used that final iterate as the starting point
for solving the optimization problem at the child node, 
we might only have to pay a constant number of SGD steps in order to 
find a solution sufficiently close to $\bb(\aa')$, instead of an exponentially
increasing (with tree height) number of SGD steps.

To formalize this intuition, and thus to design
our budget allocation strategy, we need to understand how the error of the
final SGD iterate depends on the initial distance from the optimal $\bb^*$. 
Theorem~\ref{thm:SGD_improved} is a general high probability bound on SGD iterates 
\textbf{without} assuming a global bound on
the norm of the stochastic gradient as usually done in the
literature~\cite{duchi2010composite,bubeck15,bottou2018optimization}. 
The concentration results in Theorem~\ref{thm:SGD_improved} are under similar
assumptions to the recent work in~\cite{sgdHogwild}. That work, however,
only bounds \textit{expected} error of the final iterate, not a \textit{high
probability} guarantee that we desire. 
Our bound precisely captures the dependence on the initial diameter
$d_0=\|\ww_0-\ww^*\|_2$, the global diameter bound $D$, and the noise 
floor $\GG$. This is key in designing $\lambda(h)$.
Since we are interested primarily in the scaling of error with respect to the
initial diameter, we do not emphasize the scaling of this bound with respect
to the condition number of the problem 
(our error guarantee has polynomial dependence on the condition number).
The proof of Theorem~\ref{thm:SGD_improved} is given in Appendix
\ref{sec:sgd-updated}. 

\begin{theorem}\label{thm:SGD_improved}
Consider a sequence of random samples $\{z_t\}_{t=1}^T$ drawn 
from a distribution $p(z)$
and an associated sequence
of random variables by the SGD update:
$\{\bb_{t+1} = \bb_t - \eta_t \nabla f(\bb_t; z_t)\}_{t=0}^T$, where 
$\bb_0$ is a fixed vector in $\RR^d.$
If we use the step size schedule $\eta_t = \frac{2}{\mu(t + E)}$ (where 
$E = 4096 \kappa^2 \log \Lambda^8$, $\kappa=\frac{\beta}{\mu}$, and
$\Lambda\geq t+1$),
then, under Assumptions \ref{assump1} and \ref{assumGrad}, with probability 
at least $1-\frac{t+1}{\Lambda^8}$,
the final iterate of SGD satisfies:
\begin{align*}
    \|\ww_{t+1}-\ww_0\|_2^2 &\leq \underbrace{\frac{G(d_0^2, \GG)}{t + E +
        1}}_\text{$=\EE[d_{t+1}^2]$ \cite{sgdHogwild}}\nonumber\\
    &~~~~+ \underbrace{\frac{8 (t+1)\tilde{C}(D^2, D\sqrt{\GG})}{\mu(t + 1 +
            E)\Lambda^7}}_{\substack{\text{Global diameter bound dependent,}
            \text{ controlled by $\Lambda$}}} \nonumber\\
    &~~~~+  \underbrace{\frac{4 \sqrt{2
    \log(\Lambda^8)}\sqrt{\hat{C}(k)}}{\mu(t + E +
    1)^{\alpha_{k+1}}}}_{\substack{\text{Term to control martingale deviations}
    \\ \text{Scaling is $\tilde{O}_{\varepsilon}(1/t^{\max\{1/2,
    1-\varepsilon\}})$}\text{ for any $\varepsilon>0$ }}}
\end{align*}
where $G(d_0^2, \GG) = \max\left\{ E d_0^2, \frac{8\GG}{\mu^2} \right\}$,
$\tilde{C}(D^2, D\sqrt{\GG}) = D \sqrt{8 \beta^2 D^2 + 2\GG}$,
$\hat{C}(k) = O_k(\log \Lambda^8)$, and
$\alpha_{k+1} = \sum_{i=1}^{k+1} \frac{1}{2^i}$, and $k\in\mathbb{Z}_{\geq 0}$
can be chosen as \textbf{any} nonnegative integer, and controls the scaling of 
the third term in
the above expression. Corollary  
\ref{cor:VtBoundTighter} in the Appendix gives an exact expression for 
the term $\hat{C}(k).$
\end{theorem}
 
\subsubsection{Choosing number of steps for tree search} 
Theorem \ref{thm:SGD_improved} guides our 
design of $\lambda(h)$, the budget function used by \match{} to allocate SGD
steps to nodes at height $h$.
We give full specifications of this function in Corollary \ref{cor:step} in the 
Appendix.
This Corollary shows that, as one might expect, as long as the noise of the stochastic
gradient at $\bb^*$ is sufficiently small relative to the initial distance to $\bb^*,$
then the number of steps at each node in the search tree may be chosen 
\textit{independently of tree height}. This Corollary follows immediately 
from Theorem \ref{thm:SGD_improved} and the fact that
$\| \bb_0 - \bb_{h+1,2i}^*\|_2^2 \leq 2 \|\bb_0 - \bb_{h,i}^*\|^2
+ 2\|\bb_{h,i}^* - \bb_{h+1,2i}\|_2^2 \leq 4 \nu_1 \rho^h$
by Theorem \ref{thm:claim1}.
Thus, by using a crude parent model to solve a \textit{related, but
different} optimization problem at the child node, \match{} is able to be
parsimonious with its SGD budget while still obtaining increasingly more
refined models as the search tree $\TT$ grows. 
 
\subsection{Putting it together -- bounding simple regret}
Now we present our final bound that characterizes the performance of
Algorithm~\ref{algo:mixnmatch} as Theorem~\ref{thm:tree}.
In the deterministic black-box optimization 
literature~\cite{munos2011optimistic,sen2018multi}, the quantity of interest
is generally \textit{simple regret}, $R(\Lambda)$, as defined in
(\ref{eq:simpleRegret}). 
In this line of work, the simple regret scales as a function of 
\textit{near-optimality dimension},
which is defined as follows:
\begin{definition}
The near-optimality dimension of $G(\cdot)$ with respect to parameters $(\nu_2,\rho_2)$ is given by:
$d(\nu_2,\rho_2) =  \inf \bigg\{ d'~\in~\mathbb{R}^+ : \exists~C(\nu_2,\rho_2),
    \text{ s.t. } \forall h~\geq~0,
    \mathcal{N}_{h}(3\nu_2 \rho_2^h)~\leq~C(\nu_2,\rho_2) \rho_2^{-d' h} \bigg\},$
where $\mathcal{N}_{h}(\epsilon)$ is the number of cells $(h,i)$ such that
    $\inf_{\aa \in (h,i)} G(\aa) \leq G(\aa^*) + \varepsilon $. 
\end{definition}
The near-optimality dimension intuitively states that there are 
\textit{not too many} cells which contain a point whose function values are 
\textit{close} to optimal \textit{at any tree height}.
The lower the near-optimality dimension,
the easier is the black-box optimization problem~\cite{grill2015black}. 
Theorem~\ref{thm:tree} provides a similar simple regret bound on $R(\Lambda) =
G(\aa(\Lambda)) - G(\aa^*)$, where $\aa(\Lambda)$ is the mixture weight vector
returned by the algorithm given a total SGD steps budget of $\Lambda$ and
$\aa^*$ is the optimal mixture. The proof of Theorem~\ref{thm:tree} is in
Appendix~\ref{sec:tree}.

\begin{theorem}
     \label{thm:tree}
     Let $h'$ be the smallest number $h$ such that
     $\sum_{l = 0}^{h}
     2C(\nu_2,\rho_2) \lambda(l) \rho_2^{-d(\nu_2, \rho_2)l}~>~\Lambda -
     2\lambda(h+1).$
     Then, with probability at least $1 - \frac{1}{\Lambda^3}$,
     the tree in Algorithm~\ref{algo:mixnmatch} grows to a height of at least 
     $h(\Lambda) = h'+1$ and returns a mixture weight $\aa(\Lambda)$ 
     such that
     \begin{align}
         R(\Lambda) \leq 4\nu_2\rho_2^{h(\Lambda) - 1} 
     \end{align}
\end{theorem}

Theorem~\ref{thm:tree} shows that, given a total budget of $\Lambda$ SGD
steps, \match{} recovers a mixture $\aa(\Lambda)$ with test error at most
$4\nu_2\rho_2^{h(\Lambda) - 1} $ away from the optimal test error
if we perform optimization using that mixture. The parameter $h(\Lambda)$ depends on the number of steps needed for a node expansion at different heights and crucially makes use of the fact that the starting iterate for each new node can be borrowed from the parent's last iterate. The tree search also progressively allocates more samples to deeper nodes, as we get closer to the optimum. Similar simple regret scalings have been recently shown in the context of deterministic multi-fidelity black-box optimization~\cite{sen2018multi}. We comment further on the regret scaling in 
Appendix~\ref{sec:lambdaHScaling}, ultimately noting that Theorem~\ref{thm:tree}
roughly corresponds to a regret scaling on the order of
$\tilde{O}\left(\frac{1}{\Lambda^c}\right)$ for some constant $c$ (dependent on
$d(\nu_2,\rho_2)$). 
Thus, when $|\DD_{te}|$ is much
smaller than the total computational budget $\Lambda$, our algorithm gives a
significant improvement over training only on the validation dataset.
In our experiments in Section \ref{sec:sims}
and Appendix \ref{sec:expers-appendix}, we observe that our algorithm indeed
outperforms the algorithm which trains only on the validation dataset
for several different real-world datasets.
  \section{Empirical Results}
\label{sec:sims}

We evaluate Algorithm \ref{algo:mixnmatch} against various baselines on two real-world datasets. The code used to 
create the testing infrastructure can be found at 
\url{https://github.com/matthewfaw/mixnmatch-infrastructure}, and the code (and
data) used to run experiments can be found at 
\url{https://github.com/matthewfaw/mixnmatch}.
For the simulations considered below, we divide the data into training,
validation, and testing datasets. 
 
\subsection{Experiment preliminaries}
\textbf{Algorithms compared:} We compare the following 
algorithms: (a)~\texttt{Uniform}, which trains on samples from each 
data source uniformly, (b)~\texttt{Genie}, which samples from 
training data sources according to $\aa^*$ in those cases when $\aa^*$ is 
known explicitly, (c)~\texttt{Validation}, which trains only on 
samples from the validation dataset (that is corresponding to $\aa^*$), 
(d)~\texttt{Mix\&MatchCH}, which corresponds to running \match{} by 
partitioning the $\aa$ simplex using a random coordinate halving strategy,
and (e)~\texttt{OnlyX}, which trains on samples only from data source 
\texttt{X}. We describe results with other \match{} algorithm variants in the 
supplement. 

\begin{remark}
Note that the \texttt{Genie} algorithm can be viewed as the best-case
comparison
for our algorithm in our setting.
Indeed, any algorithm which aims to find the data distribution for the
validation dataset will, in the best case, find the true mixture $\aa^*$ by
Proposition~\ref{prop:validationLoss}. Given $\aa^*$, the model minimizing
validation loss may be obtained by running SGD on this mixture distribution
over the training datasets.
Thus, the \texttt{Genie} AUROC scores can be viewed as an upper bound
for the achievable scores in our setting.
\end{remark}

\textbf{Models and metrics:} We use fully connected 3-layer neural networks
with ReLU activations for all our experiments, training
with cross-entropy loss on the categorical labels. We use the test AUROC
as the metric for comparison between the above mentioned algorithms.
For multiclass problems, we use multiclass AUROC metric described in
\cite{Hand2001}.
The reason for using AUROC is due to the label imbalances due to covariate
shifts between the training sources  and our test and validation sets.
In all the figures displayed, each data point is a result of averaging over
10 experiments with the error bars of $1$ standard deviation.
Note that while all error bars are displayed
for all experiments, some error bars are too small to see in the plots. 
 
\subsection{Allstate Purchase Prediction Challenge:}
The Allstate Purchase 
Prediction Challenge Kaggle dataset \cite{kaggle} has
entries from customers across different states in the US. The goal is to
predict what option a customer would choose for an insurance plan in a specific
category (Category G with $4$ options ). The dataset features include 
(a)~demographics and details regarding vehicle ownership of a customer and 
(b)~timestamped information about insurance plan selection across seven
categories (A-G) used by customers to obtain price quotes.
There are multiple timestamped category selections and corresponding price
quotes for a customer. We collapse the selections and the price quote to a
single set of entries using summary statistics of the time stamped features.
\begin{figure}[ht]
    \centering
    \includegraphics[scale=0.4]{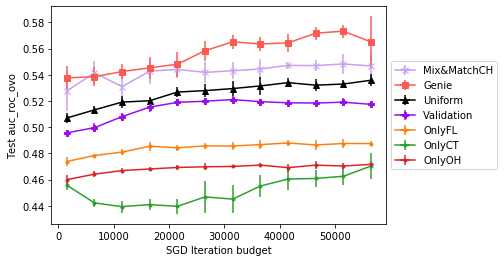}
    \caption{\small Test AUROC for predicting insurance plan for a mixture
    of FL, CT, and OH data}
    \label{fig:allstate3auc}
\end{figure}

In this experiment, we split the Kaggle dataset into
$K=3$ training datasets correspond to customer data from three states:
Florida (FL),
Connecticut (CT), and Ohio (OH). The validation and test datasets 
also consist of customers from these states, but the proportion of
customers from various states is fixed. Details about the test and validation
set formation is in the Appendix. 
In this case, $\aa^*$ is explicitly known for the \texttt{Genie} algorithm.

As shown in Figure \ref{fig:allstate3auc},
with respect to the AUROC metric, \texttt{Mix\&MatchCH} is competitive with
the \texttt{Genie} algorithm and has superior performance to all other 
baselines. The \texttt{Validation} algorithm has performance inferior to the
uniform sampling scheme. Therefore, we are operating in a regime in which
training on the validation set alone is not sufficient for good performance.
 
\subsection{Amazon Employee Access Challenge:}
\begin{figure}[ht]
    \centering
    \includegraphics[scale=0.4]{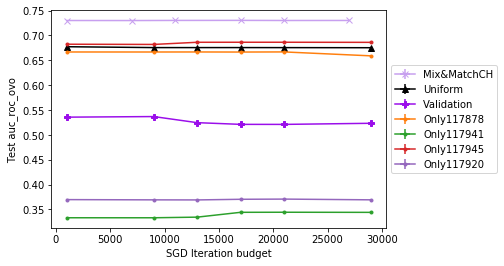}
    \caption{\small Test AUROC for predicting employee access in a new department, using training data from 4 departments}
    \label{fig:amazon}
\end{figure}
We evaluate our algorithms on the Amazon Employee Access Challenge Dataset 
\cite{amazon}. The goals is to whether or not the employee is allowed to
access a resource given details about the employees role in the organization.
We split the training data into different sources based on departments.
The validation and test set has data from a new department unseen in the 
training data sources (In this case we don't know $\aa^*$ explicitly to evaluate
the \texttt{Genie} Algorithm). Additional details about the formation
of datasets is in the Appendix.

We find that \texttt{Mix\&MatchCH}
outperforms the other baselines, and 
training solely on validation is insufficient to obtain a good AUROC score.

\section{Acknowledgements}

This work was partially supported by NSF Grant SATC 1704778, 
ARO grant W911NF-17-1-0359 and the WNCG Industrial Affiliated Program.

\bibliography{MixnMatch}

\clearpage

\appendix

\section{A More Detailed Discussion on Prior Work}
\label{sec:rwork-long}
Transfer learning has assumed an increasingly important role, especially in settings where we are either computationally limited, or data-limited, and yet we have the opportunity to leverage significant computational and data resources yet on domains that differ slightly from the target domain \cite{raina2007self,pan2009survey,dai2009eigentransfer}. This has become an important paradigm in neural networks and other areas \cite{yosinski2014transferable,oquab2014learning,bengio2011deep,kornblith2018better}. 

An important related problem is that of covariate shift \cite{shimodaira2000improving,zadrozny2004learning,gretton2009covariate}. The problem here is that the target distribution may be different from the training distribution. A common technique for addressing this problem is by reweighting the samples in the training set, so that the distribution better matches that of the training set. There have been a number of techniques for doing this. An important recent thread has attempted to do this by using unlabelled data \cite{huang2007correcting,gretton2009covariate}. Other approaches have considered a related problem of solving a weighted log-likelihood maximization \cite{shimodaira2000improving}, or by some form of importance sampling \cite{sugiyama2007covariate,sugiyama2008direct} or bias correction \cite{zadrozny2004learning}. In~\cite{mohri2019agnostic}, the authors study a related problem of learning from different datasets, but provide mini-max bounds in terms of an agnostically chosen test distribution.

Our work is related to, but differs from all the above.  As we explain in Section \ref{sec:psetting}, we share the goal of transfer learning: we have access to enough data for training, but from a family of distributions that are different than the validation distribution (from which we have only enough data to validate). Under a model of covariate shift due to unobserved variables, we show that a target goal is finding an optimal reweighting of populations rather than data points. We use optimistic tree search to address precisely this problem -- something that, as far as we know, has not been undertaken. 

A key part of our work is working under a computational budget, and then designing an optimistic tree-search algorithm under uncertainty. We use a single SGD iteration as the currency denomination of our budget -- i.e., our computational budget requires us to minimize the number of SGD steps in total that our algorithm computes. Enabling MCTS requires a careful understanding of SGD dynamics, and the error bounds on early stopping. There have been important SGD results studying early stopping, e.g., \cite{hardt2015train,bottou2018optimization} and generally results studying error rates for various versions of SGD and recentered SGD \cite{sgdHogwild,defazio2014saga,roux2012stochastic}. Our work requires a new high probability bound, which we obtain in the Supplemental material, Section \ref{sec:sgd-updated}. In \cite{sgdHogwild}, the authors have argued that a uniform norm bound on the stochastic gradients is not the best assumption, however the results in that paper are in expectation. In this paper, we derive our SGD high-probability bounds under the mild assumption that the SGD gradient norms are bounded only at the optimal weight $\bb^*$. 

There are several papers \cite{harvey2018tight,rakhlin2012making} 
which derive high probability bounds on the 
suffix averaged and final
iterates returned by SGD for non-smooth strongly convex functions. However,
both papers operate under the assumption of uniform bounds on the stochastic
gradient. Although these papers do not directly report a dependence on
the diameter of the space, since they both consider projected gradient descent,
one could easily translate their constant dependence to a sum of a diameter
dependent term and a stochastic noise term 
(by using the bounded gradient assumption from \cite{sgdHogwild}, for example).
However, as the set into which the algorithm would project is unknown to our 
algorithm (i.e., it would require knowing $\bb^*$), we cannot use projected
gradient descent in our analysis. As we see in later sections, we need a high-probability SGD guarantee
which characterizes the dependence on diameter of the space and noise of the
stochastic gradient. It is not immediately clear how the analysis in
\cite{harvey2018tight, rakhlin2012making} could be extended in this setting
under the gradient bounded assumption in \cite{sgdHogwild}. In 
Section~\ref{sec:results}, we instead 
develop the high probability bounds that are needed in our setting.

Optimistic tree search makes up the final important ingredient in our algorithm. These ideas have been used in a number of settings \cite{bubeck2011x,grill2015black}. Most relevant to us is a recent extension of these ideas to a setting with biased search \cite{sen2018multi,sen2019noisy}. 
 \section{Standard Definitions from Convex Optimization}
\label{sec:defs}

Recall that we assume throughout the paper that our loss functions satisfy the
following assumptions similar to \cite{sgdHogwild}:
\begin{repassumption}{assump1}[Restated from main text]
For each
loss function $f(\cdot; z)$ corresponding to a \textbf{sample}
$z\in\mathcal{Z}$,
we assume that $f(\cdot; z)$ is: \emph{(i)}~\textbf{$\beta$-smooth} (Definition
\ref{def:smooth})
and \emph{(ii)}~\textbf{convex} (Definition \ref{def:convex}).

Additionally, we assume that, for each $\aa\in\AA$, the \textbf{averaged}
loss function
$\Fa(\cdot)$ is: \emph{(i)}~\textbf{$\mu$-strongly convex} (Definition
\ref{def:strongConvex})
and \emph{(ii)}~\textbf{$L$-Lipschitz} (Definition \ref{def:lips}).
\end{repassumption}

We now state the definitions of these notions, which are standard in the
optimization literature (see, for example, \cite{bubeck15}).

\begin{definition}[$L$-Lipschitz]\label{def:lips}
    We call a function $g(\cdot)$ $L$-Lipschitz if, for all $\ww,\ww'\in\WW$,
    \begin{align*}
        |g(\ww)-g(\ww')| \leq L\|\ww-\ww'\|_2.
    \end{align*}
\end{definition}
\begin{definition}[$\beta$-smooth]\label{def:smooth}
    We call a function $g(\cdot)$ $\beta$-smooth when, for all
    $\ww,\ww'\in\WW$ when the gradient of $f$ is $\beta$-Lipschitz, i.e.,
    \begin{align*}
        \|\nabla g(\ww) - \nabla g(\ww')\|_2 \geq \beta\|\ww-\ww'\|_2.
    \end{align*}
\end{definition}
\begin{definition}[Convex]\label{def:convex}
    We call a function $g(\cdot)$ convex when, for all $\ww,\ww'\in\WW$,
    \begin{align*}
        g(\ww) \geq g(\ww') + \<\nabla g(\ww'), \ww-\ww'\>.
    \end{align*}
\end{definition}
\begin{definition}[$\mu$-strongly convex]\label{def:strongConvex}
    We call a function $g(\cdot)$ $\mu$-strongly convex if, for all
    $\ww,\ww'\in\WW$,
    \begin{align*}
        g(\ww) \geq g(\ww') + \<\nabla g(\ww'), \ww-\ww'\> +
        \frac{\mu}{2}\|\ww-\ww'\|_2^2.
    \end{align*}
\end{definition}

 \section{Smoothness with Respect to \texorpdfstring{$\aa$}{}}
\label{sec:claim1}
In this section we prove Theorem~\ref{thm:claim1}. The analysis is an interesting generalization of Theorem 3.9 in~\cite{hardt2015train}. The key technique is to create a total variational coupling between $\aa_1$ and $\aa_2$. Then using this coupling we prove that SGD iterates from the two distributions cannot be too far apart in expectation. Therefore, because the two sets of iterates converge to their respective optimal solutions, we can conclude that the optimal weights $\bb^*(\aa_1)$ and $\bb^*(\aa_2)$ are close. 

\begin{lemma}\label{lem:alphaOmegaRelation}
 Under conditions of Theorem~\ref{thm:claim1}, let $\bb_n(\aa_1)$ and $\bb_n(\aa_2)$ be the random variables representing the weights after performing $n$ steps of online projected SGD onto a convex body $\mathcal{X} = \mathrm{Conv}\{\bb^*(\aa) \mid \aa\in\mathcal{A}\}$ using the data distributions represented by the mixtures $\aa_1$ and $\aa_2$ respectively, starting from the same initial weight $\bb_0$, and using the step size sequence described in Theorem \ref{thm:SGD_improved}. Then we have the following bound,
\begin{align*}
\EE \left[ \norm{\bb_n(\aa_1) - \bb_n(\aa_2)}\right] \leq \frac{2 \sigma \norm{\aa_1 - \aa_2}_1}{\mu}.
\end{align*}
where $\sigma^2 = \sup_{\bb, \bb'\in\mathcal{X}}\sup_{\aa\in\mathcal{A}} \beta^2 \|\bb - \bb'\|^2 + \GG(\aa).$
\end{lemma}

\begin{proof}
We closely follow the proof of Theorem 3.9 in~\cite{hardt2015train}. Let $\bb_{t+1}(\aa_i) = \Pi_\mathcal{X}(\bb_t - \eta_t \nabla f(\bb_t; Z^{(i)}_t))$ denote the SGD update while processing the $t$-th example from $\aa_i$ for $i\in\{1,2\}$. Let $I,J$ be two random variables whose joint distribution follows the variational coupling between $\aa_1$ and $\aa_2$. Thus the marginals of $I$ and $J$ are $\aa_1$ and $\aa_2$ respectively, while $\PP(I \neq J) = d_{TV}(\aa_1,\aa_2)$. At each time $I_t \sim I$ and $J_t \sim J$ are drawn. If $I_t = J_t$, then we draw a data sample $Z_t$ from $D_{I_t}$ and set $Z^{(1)}_t = Z^{(2)}_t = Z_t$. Otherwise, we draw $Z^{(1)}_t$ from $D_{I_t}$ and $Z^{(2)}_t$ from $D_{J_t}$ independently.

Therefore, following the analysis in~\cite{hardt2015train}, if $I_t = J_t$, then, by 
Lemma 3.7.3 in \cite{hardt2015train}, by our choice of step size, and since Euclidean projection does not increase the distance between projected points (see for example Lemma 3.1 in \cite{bubeck15}),
\begin{align*}
\delta_{t+1}^2 &= \| \bb_{t+1}(\aa_1) - \bb_{t+1}(\aa_2) \|^2\\
&= \| \Pi_\mathcal{X}(\bb_t(\aa_1) - \eta_t \nabla f(\bb_t(\aa_1); Z_t)) - \Pi_\mathcal{X}(\bb_t(\aa_2) - \eta_t \nabla f(\bb_t(\aa_2); Z_t)) \|^2\\
&\leq \| \bb_t(\aa_1) - \eta_t \nabla f(\bb_t(\aa_1); Z_t) - \bb_t(\aa_2) + \eta_t \nabla f(\bb_t(\aa_2; Z_t) \|^2\\
&= \delta_t^2 + \eta_t^2 \| \nabla f(\bb_t(\aa_1); Z_t) - \nabla f(\bb_t(\aa_2; Z_t) \|^2\\
&~~~~- 2 \eta_t \<\nabla f(\bb_t(\aa_1); Z_t) - \nabla f(\bb_t(\aa_2); Z_t), \bb_t(\aa_1) - \bb_t(\aa_2) \> \\
\end{align*}
Now, taking expectations with respect to $Z_t$, we get the following:
\begin{align*}
\EE_{Z_t}[\delta_{t+1}^2] &\leq
 \delta_t^2 + \eta_t^2 \EE_{Z_t}\| \nabla f(\bb_t(\aa_1); Z_t) - \nabla f(\bb_t(\aa_2); Z_t) \|^2\\
&~~~~- 2 \eta_t \<\nabla F(\bb_t(\aa_1)) - \nabla F(\bb_t(\aa_2)), \bb_t(\aa_1) - \bb_t(\aa_2) \> \\
&\leq \delta_t^2 + \eta_t^2 \underbrace{\| \nabla f(\bb_t(\aa_1); Z_t) - \nabla f(\bb_t(\aa_2; Z_t) \|^2}_{\leq \beta^2 \delta_t^2 \text{ by smoothness of $f$}}\\
&~~~~- 2 \eta_t \underbrace{\left(\frac{\mu\beta}{\mu + \beta} \|\bb_t(\aa_1) - \bb_t(\aa_2) \|^2 + \frac{1}{\mu + \beta} \underbrace{\|\nabla F(\bb_t(\aa_1)) - \nabla F(\bb_t(\aa_2)) \|^2}_{\leq \mu^2\delta_t^2 \text{by strong convexity of $F$}} \right)}_\text{bound holds by Lemma 3.11 in \cite{bubeck15}}\\
&= \left(1 - 2\eta_t \frac{\mu \beta}{\mu + \beta}\right)\delta_t^2 - \eta_t\underbrace{\left(\frac{2\mu^2}{\mu + \beta} - \beta^2\eta_t\right)}_{\geq 0 \text{ by choice of $\eta_t$}} \delta_t^2\\
&\leq (1 - \mu \eta_t)\delta_t^2
\end{align*}
where the last inequality follows since $\eta_t < \frac{\mu + \beta}{2\mu\beta}$,
$\eta_t \leq \frac{2}{\kappa^2(\mu + \beta)} \leq \frac{1}{\mu \kappa^2},$ and $\beta \geq \mu.$
Thus, when $I_t = J_t,$ we have that
\begin{align*}
    \EE_{Z_t}[\delta_{t+1}] 
    &\leq \sqrt{\EE_{Z_t}[\delta_{t+1}^2]} && \text{Using concavity of $\sqrt{.}$, and applying Jensen's inequality}\\
    &\leq \sqrt{(1 - \mu \eta_t)\delta_t^2} && \text{using our bound above}\\
    &\leq (1 - \mu \eta_t/2) \delta_t && \text{since $\sqrt{1 - \mu\eta_t} \leq 1 - \frac{\mu \eta_t}{2}$.}
\end{align*}

On the other hand, when $I_t\neq J_t$, we have that 
\begin{align*}
    \delta_{t+1} &\leq \| \bb_t(\aa_1) - \eta_t \nabla f(\bb_t(\aa_1; Z^{(1)}_t) - \bb_t(\aa_2) + \eta_t \nabla f(\bb_t(\aa_2; Z^{(2)}_t)) \|\\
    &\leq \| \bb_t(\aa_1) - \eta_t \nabla f(\bb_t(\aa_1; Z^{(1)}_t) - \bb_t(\aa_2) + \eta_t \nabla f(\bb_t(\aa_2; Z^{(1)}_t) \|\\
    &~~~~+ \eta_t(\| \nabla f(\bb(\aa_2); Z^{(1)})\| + \|\nabla f(\bb(\aa_2); Z^{(2)}) \|)\\
    &\leq (1 - \mu \eta_t/2)\delta_t + \eta_t (\| \nabla f(\bb(\aa_2); Z^{(1)})\| + \|\nabla f(\bb(\aa_2); Z^{(2)}) \|) && \text{By the previous bound}\\
    &\leq (1 - \mu \eta_t/2)\delta_t + \eta_t \sqrt{2\beta^2 \|\bb_t(\aa_2) - \bb^*(\aa_1)\|^2 + 2 \GG(\aa_1)}\\
    &~~~~+ \eta_t\sqrt{2\beta^2 \|\bb_t(\aa_2) - \bb^*(\aa_2)\|^2 + 2 \GG(\aa_2)}) && \text{by Lemma \ref{lem:samplePathLem2}}\\
    &\leq(1 - \mu \eta_t/2)\delta_t + 2 \sigma \eta_t
\end{align*}
where $\sigma^2 = \sup_{\bb, \bb'\in\mathcal{X}}\sup_{\aa\in\mathcal{A}} \beta^2 \|\bb - \bb'\|^2 + \GG(\aa).$

Thus, by combining both of these results, we obtain:
\begin{align*}
\EE[\delta_{t+1}] &\leq (1 - \mu\eta_t/2) \EE[\delta_{t}] + 2\sigma \eta_t \PP\{I_t \neq J_t\} \\
& = (1 - \mu\eta_t/2) \EE[\delta_{t}] +  \sigma\eta_t \norm{\aa_1 - \aa_2}_1.
\end{align*}

Assuming that $\delta_{t_0} = 0$, we get the following result from the recursion,
\begin{align*}
\EE[\delta_n] &\leq \sum_{t = t_0}^{n} \left\{\prod_{s = t+1}^{n} \left(1 - \frac{1}{s + E}\right)\right\} \frac{2 \sigma}{\mu (t + E)} \norm{\aa_1 - \aa_2}_1 \\
&= \sum_{t = t_0}^{n} \frac{t + E}{n + E} \frac{2 \sigma}{\mu (t + E)} \norm{\aa_1 - \aa_2}_1 \\
&\leq \frac{n - t_0 + 1}{n + E} \frac{2 \sigma}{\mu } \norm{\aa_1 - \aa_2}_1 \\
& \leq \frac{2 \sigma}{\mu } \norm{\aa_1 - \aa_2}_1.
\end{align*}

\end{proof}

\begin{proof}[Proof of Theorem~\ref{thm:claim1}]
First, note that by definition $\bb^*(\aa)$ is not a random variable i.e it is the optimal weight with respect to the distribution corresponding to $\aa$. On the other hand, $\bb_n(\cdot)$
is a random variable,
where the randomness is coming from the randomness in SGD sampling. By the triangle inequality, we have the following:
\begin{align}
\norm{\bb^*(\aa_1) - \bb^*(\aa_2)}  &\leq \norm{\bb^*(\aa_1) - {\bb}_n(\aa_1)} + \norm{{\bb}_n(\aa_1) - {\bb}_n(\aa_2)} + \nonumber\\
&\norm{\bb^*(\aa_2) - {\bb}_n(\aa_2)} \nonumber\\
\implies \norm{\bb^*(\aa_1) - \bb^*(\aa_2)} &= \EE[\norm{\bb^*(\aa_1) - \bb^*(\aa_2)}] \nonumber\\
&\leq \EE[\norm{\bb^*(\aa_1) - {\bb}_n(\aa_1)}] + \EE[\norm{{\bb}_n(\aa_1) - {\bb}_n(\aa_2)}] \nonumber\\
&+ \EE[\norm{\bb^*(\aa_2) - {\bb}_n(\aa_2)}].\label{eq:thm1ineq}
\end{align}
The expectation in the middle of the r.h.s.\ is bounded as in Lemma \ref{lem:alphaOmegaRelation}. 
We can use Theorem 2 in \cite{sgdHogwild} and Jensen's inequality to bound the other two 
terms on the r.h.s. as
\begin{align*}
    \EE[\| {\bb}_n(\aa_1) - \bb^*(\aa_1) \|_2 ] &\leq \sqrt{\EE[\| {\bb}_n(\aa_1) - \bb^*(\aa_1) \|_2^2 ]} && \text{by concavity of $\sqrt{.}$}\\
    &\leq \sqrt{\frac{2 \GG G}{\mu^2(n + E)}} && \text{by Theorem 2 in \cite{sgdHogwild}\footnotemark}
\end{align*}
\footnotetext{Note that here, we are considering \textit{projected} SGD, while the analysis in \cite{sgdHogwild} is done without projection. Note that the proof of Theorem 2 trivially continues to hold under projection, as a result of the inequality $\| \Pi_{\mathcal{X}}(\tilde{\bb}_{t+1}) - \bb^* \|^2 \leq \| \tilde{\bb}_{t+1} - \bb^* \|^2$ (see Lemma 3.1 in \cite{bubeck15}), for example.}
where we take $\GG = \max\{\GG(\aa_1),\GG(\aa_2)\}$, $E$ is chosen as in Theorem \ref{thm:SGD_improved}, and $G = \max\left\{ \frac{E\mu^2}{2\GG}, 4 \right\}.$
Now, noting that the inequality (\ref{eq:thm1ineq}) holds for all $n$, we have the bound claimed in
Theorem \ref{thm:claim1}.
\end{proof}

\begin{proof}[Proof of Corollary~\ref{cor:alphaGeometricDecay}]
This proof is a straightforward consequence of Theorem 3.1 in \cite{simplexBisection} and Theorem \ref{thm:claim1}. In 
particular, Theorem 3.1 in \cite{simplexBisection} tells us that under the method of bisection of the simplex which they describe, 
\begin{align*}
    \| \aa_1 - \aa_2 \|_2 \leq \left(\frac{\sqrt{3}}{2}\right)^{\lfloor \frac{h}{K-1} \rfloor} \text{diam}(\AA),
\end{align*}
where $\text{diam}(\AA) = \sup\{ \| \aa - \aa' \|_2 \mid \aa, \aa' \in \AA\},$ and $K-1 = \mathrm{dim}(\AA).$ As noted
in Remark 2.5 in \cite{simplexBisection}, $\text{diam}(\AA) = \sqrt{2}$ since $\AA$ is
the unit simplex. Thus, by the Cauchy-Schwartz inequality, and since $\left\lfloor \frac{h}{K-1} \right\rfloor > \frac{h}{K-1} - 1$, we have the following:
\begin{align*}
    \| \aa_1 - \aa_2 \|_1 &\leq \sqrt{K} \| \aa_1 - \aa_2 \|_2\\
    &\leq \sqrt{2 K} \left(\frac{\sqrt{3}}{2} \right)^{\lfloor \frac{h}{K-1} \rfloor}\\
    &\leq \sqrt{2 K} \left(\frac{\sqrt{3}}{2} \right)^{\frac{h}{K-1} - 1}.
\end{align*}
Now we may use this result along with our assumption that $F$ is Lipschitz and Theorem \ref{thm:claim1} to obtain:
\begin{align*}
    |G(\aa_1) - G(\aa_2)| &= | F^{(te)}(\bb^*(\aa_1)) - F^{(te)}(\bb^*(\aa_2)) |\\
    &= | \Fas(\bb^*(\aa_1)) - \Fas(\bb^*(\aa_2)) |\\
    &\leq L \| \bb^*(\aa_1) - \bb^*(\aa_2) \|_2\\
    &\leq \frac{2 L \sigma \| \aa_1 - \aa_2 \|_1}{\mu}\\
    &\leq \frac{4 L \sigma \sqrt{2 K}}{\sqrt{3} \mu} \left( \frac{\sqrt{3}}{2} \right)^{\frac{h}{K-1}},
\end{align*}
which is the desired result.
\end{proof} \section{New High-Probability Bounds on SGD without a Constant Gradient Bound}
\label{sec:sgd-updated}

In this section, we will prove a high-probability bound on any
iterate of SGD evolving over the time interval $t = 1, 2, \ldots, T,$ without assuming a uniform bound on the stochastic gradient over the domain. Instead, this bound introduces a tunable parameter $\Lambda > (T+1)$ that controls the trade-off between a bound on the SGD iterate $d_t^2,$ and the probability with which the bound holds. As we discuss in Remark~\ref{remark:unifDiamBound}, this parameter can be set to provide tighter high-probability guarantees on the SGD iterates in settings where the diameter of the domain is large and/or cannot be controlled.

\begin{reptheorem}{thm:SGD_improved}[Restated from main text]
Consider a sequence of random samples $z_0, z_1,\ldots, z_T$ drawn from a distribution $p(z).$
Define the filtration $\mathcal{F}_t$ generated by $\sigma\{z_0, z_1,\ldots, z_t\}.$ Let us define a sequence
of random variables by the gradient descent update: $\bb_{t+1} = \bb_t - \eta_t \nabla f(\bb_t; z_t), t =1, \ldots, T$, and $\bb_0$ is a fixed vector in $\RR^d.$

If we use the step size schedule $\eta_t = \frac{2}{\mu(t + E)},$ where 
$E = 4096 \kappa^2 \log \Lambda^8,$
then, under Assumptions \ref{assump1} and \ref{assumGrad}, and taking $\Lambda \geq t + 1$, we have the following high probability bound on the final iterate of the SGD procedure after $t$ time steps for
\textbf{any} $k\geq 0$:
\begin{align}
    \mathrm{Pr}\left(d_{t+1}^2 > \frac{G(d_0^2, \GG)}{t + E + 1} + \frac{8 (t+1)\tilde{C}(D^2, D\sqrt{\GG})}{\mu(t + 1 + E)\Lambda^7} +  \frac{4 \sqrt{2 \log(\Lambda^8)}\sqrt{\hat{C}(k)}}{\mu(t + E + 1)^{\alpha_{k+1}}}
 \right) &\leq \frac{t + 1}{\Lambda^8}
\end{align}
where
\begin{align*}
    G(d_0^2, \GG) &= \max\left\{ E d_0^2, \frac{8\GG}{\mu^2} \right\}\\
    \tilde{C}(D^2, D\sqrt{\GG}) &= D \sqrt{8 \beta^2 D^2 + 2\GG}\\
    \hat{C}(k) &= O(\log \Lambda^8) & \text{See Corollary \ref{cor:VtBoundTighter} for definition and Remark \ref{remark:C-scaling} for discussion.}\\
    \alpha_{k+1} &= \sum_{i=1}^{k+1} \frac{1}{2^i}.
\end{align*}
In particular, when we choose $k=0$, the above expression becomes
\begin{align}
    \mathrm{Pr}\left(d_{t+1}^2 > \frac{G(d_0^2, \GG)}{t + E + 1} + \frac{8 (t+1)\tilde{C}(D^2, D\sqrt{\GG})}{\mu(t + 1 + E)\Lambda^7} +  \frac{4 \sqrt{2\check{C}} \log(\Lambda^8)}{\mu\sqrt{t + E + 1}}
 \right) &\leq \frac{t + 1}{\Lambda^8}
\end{align}
where 
\begin{align*}
    \check{C} &= \max\left\{ \frac{8 d_0^2(4\beta^2 d_0^2 + \GG)}{(1+E)\log\Lambda^8}, \left(\frac{32\sqrt{2}\GG}{\mu} + \frac{2}{E}\right)^2, \left( \frac{64 \beta^2 c_1^2}{(1+E)\log \Lambda^8} + \frac{8 \GG c_1}{\log \Lambda^8}\right)^2 \right\}\\
    c_1 &= G(d_0^2, \GG) + \frac{8 \tilde{C}(D^2, D\sqrt{\GG})}{\mu \Lambda^6}
\end{align*}
\end{reptheorem}

\begin{remark}
\label{remark:terms-in-theorem}
This result essentially states that the distance of $\bb_t$ to $\bb^*$ is at most the sum of three terms with high probability. Recall from the first step of the proof of Theorem 2 in \cite{sgdHogwild} that $\EE[d_t^2] \leq (1-\mu\eta_t)\EE[d_t^2] + 2\eta_t^2\GG + \EE[M_t],$ where $M_t = \< \nabla F(\bb_t) - \nabla f(\bb_t; z_t), \bb_t - \bb^*\>$ is a martingale difference sequence with respect to the filtration generated by samples $\bb_0,\ldots,\bb_t$ (in particular, note that $\EE[M_t]=0$). We obtain a similar inequality in the high probability analysis without the expectations,
so bounding the $M_t$ term is the main difficulty in proving the high probability convergence guarantee.
Indeed, the first term in our high-probability guarantee corresponds to a bound on the $(1-\mu\eta_t)d_t + 2\eta_t \GG$ term. Thus, as in the expected value analysis from \cite{sgdHogwild}, this term decreases linearly
in the number of steps $t$, with the scaling constant depending only on the initial distance $d_0$
and a uniform bound on the stochastic gradient at the optimum model parameter ($\bb^*$).

The latter two terms correspond to a bound on a normalized version of the martingale $\sum_{i}M_i$, which appears after unrolling the aforementioned recursion. Due to our
more relaxed assumption on the bound on the norm of the stochastic gradient, we employ
different techniques in bounding this term than were used in \cite{harvey2018tight}.
The second term is a bias term that depends on the worst case diameter bound $D$ (or if no diameter bound exists, then $D$  represents the worst case distance between $\bb_t$ and $\bb^*,$ see Remark~\ref{remark:unifDiamBound}), and appears as a result of applying Azuma-Hoeffding with conditioning. Our bound exhibits a trade-off between the bias term which is $O(D^2/\mathrm{poly}(\Lambda)),$ and the probability of the bad event which is $\frac{t+1}{\Lambda^8}.$ This trade-off can be achieved by tuning the parameter $\Lambda.$ 
Notice that while the probability of the bad event decays polynomially in $\Lambda,$
the bias only increases as $\mathrm{poly}(\log \Lambda).$ 

The third term represents the deviation of the martingale,
which decreases nearly linearly in $t$ (i.e. $t^\gamma$ for any $\gamma$ close to $1$). The scaling constant, however, depends on $\gamma$. By choosing $\Lambda$ appropriately (in the second term), this third term decays the slowest of the three, for large values of $t,$ and is thus the most important one from a scaling-in-time perspective.

\end{remark}

\begin{remark}\label{remark:unifDiamBound}
In typical SDG analysis (e.g. \cite{duchi2010composite, harvey2018tight}), a uniform bound on the stochastic gradient
is assumed. Note that if we assume a uniform bound on $d_t$, i.e. $d_t \leq D ~\forall~ t\in[1,T],$ then
under Assumption \ref{assump1}, we immediately obtain a uniform bound on the stochastic gradient, since:
\begin{align}
    \| \nabla f(\bb_t; z) \| &\leq \| \nabla f(\bb_t; z) - \nabla f(\bb^*; z) \| + \| \nabla f(\bb^*; z) \| \nonumber\\
    &\leq \beta d_t + \sqrt{\GG}\nonumber\\
    &\leq \beta D + \sqrt{\GG} := \sqrt{\Bar{\mathcal{G}}}
\end{align}
If we do not have access to a projection operator on our feasible set of $\bb,$ or otherwise choose
not to run projected gradient descent, then we obtain a worst-case upper bound of
$D = O\left( t^{u} \right)$ where $u = 2\sqrt{2}\kappa^{3/2}$, since:
\begin{align*}
    d_{t+1} &\leq d_t + \eta_t \| \nabla f(\bb_t; z_t) \| && \text{by triangle inequality and definition of the SGD step}\\
    &\leq d_t + \eta_t \sqrt{2 \beta^2 \kappa d_t^2 + 2\GG} && \text{by Lemma \ref{lem:samplePathLem2}}\\
    &\leq \left(1 + \frac{\alpha\sqrt{2\kappa}\beta}{\mu(t + E)}\right)d_t + \frac{\alpha\sqrt{2\GG}}{\mu(t + E)} && \text{by choice of $\eta_t = \frac{\alpha}{\mu(t+E)}$, where $\alpha > 1$ must hold}\\
    &= O\left(t^{\alpha\sqrt{2}\kappa^{3/2}}\right) && \text{we take $\alpha=2$ throughout this paper}
\end{align*}
Thus, when we do not assume access to the feasible set of $\bb$ and do not run projected gradient
descent, a convergence guarantee of the form $\tilde{O}\left( \frac{\Bar{\mathcal{G}}}{t} \right)$ that follows from a uniform bound on the stochastic gradient
does not suffice in our setting because $\Bar{\mathcal{G}}$ scales polynomially in $t$. We further note that even if we do have access to a projection operator, $\Bar{\mathcal{G}}$ scales quadratically in the radius of the projection set, and thus can be very large. 

Instead, we wish to construct a high probability guarantee on the final SGD iterate in a fashion similar to
the expected value guarantee given in \cite{sgdHogwild}.
Now under our construction, we have an additional parameter, $\Lambda,$ which we may use to our advantage 
to obtain meaningful convergence results \textit{even when $D$ scales polynomially}. Indeed, we observe that each occurrence of $\tilde{C}$ in our construction is normalized by at least $\Lambda^2.$
Thus, since $\tilde{C} = O(D^2),$ 
by replacing $\Lambda \leftarrow \Lambda^{2u + 1}$ in our analysis, and assuming $\Lambda$ is polynomial in $t,$
we can obtain (ignoring polylog factors) $\tilde{O}\left(\frac{1}{t^{\gamma}}\right)$ convergence of the 
final iterate of SGD, for any $\gamma < 1$.
Note that this change simply modifies the definition of $r_t$ by a constant factor. Thus, our
convergence guarantee continues to hold with minor modifications to the choice of constants in our analysis.

\end{remark}

A direct consequence of Theorem \ref{thm:SGD_improved} and the fact that $\| \bb_0 - \bb_{h+1, 2i}^*\|_2^2 \leq 2 \| \bb_0 - \bb_{h,i}\|_2^2 + 2 \| \bb_{h,i}^* - \bb_{h+1,2i}\|_2^2 \leq 4 \nu_1 \rho^h$ by Theorem \ref{thm:claim1} is the following Corollary, which guides our SGD budget allocation strategy.

\begin{corollary}
\label{cor:step}
Consider a tree node $(h,i)$ with mixture weights $\aa_{h,i}$ and optimal learning parameter $\bb^*_{h,i}$. Assuming we start at a initial point $\bb_0$ such that $\norm{\bb_0 - \bb^*_{h,i}}_2^2 \leq \nu_1\rho^h$ and take $t = \lambda(h+1)$ SGD steps using the child node distribution $p^{(\aa^*_{h+1,2i})}$ where, $\lambda(h+1)$ is chosen to satisfy
\begin{align}
    \frac{G(4\nu_1 \rho^h, \GG)}{\lambda(h+1) + E} + \frac{8 \lambda(h+1) \tilde{C}(D^2, D\sqrt{\GG})}{\mu(\lambda(h+1) + E)\Lambda^7} + \frac{4\sqrt{2 \log(\Lambda^8)} \sqrt{\hat{C}(k)}}{\mu(\lambda(h+1) + E)^{\alpha_{k+1}}} \leq \nu_1 \rho^{h+1},
\end{align}
then by Theorem \ref{thm:SGD_improved}, 
with probability at least $1 - \frac{1}{\Lambda^7}$ we have $||\bb_t - \bb^*_{h+1,2i}||_2^2 \leq \nu_1\rho^{h+1}.$ 

In particular, if we assume that $D^2 = K(t) d_0^2$ for some $K(t)$ such that $K(t)/\Lambda^6 = \hat{K} = O(1)$ 
(refer to Remark \ref{remark:unifDiamBound} for why this particular assumption is reasonable), 
then when $G = E d_{0}^2$ (i.e. $E d_0^2 \geq \frac{8\GG}{\mu}$) and 
$\hat{C}(0) = \frac{8 d_0^2(4 \beta^2 d_0^2 + \GG)}{1 + E}$ (note that a similar statement can be made if the third term inside the max in $\check{C}$ from Theorem \ref{thm:SGD_improved}, instead of the first term, is maximal), 
taking $k=0$, we may choose $\lambda(h)$ \textbf{independently of $h$}:
\begin{align}
\lambda(h+1) = \lambda &= \Bigg( \frac{1}{\rho \sqrt{1 + E}}\bigg( 4E + 64\sqrt{2}\kappa \hat{K} + \frac{16 \sqrt{E \hat{K}}}{\sqrt{\mu} \Lambda^3} \nonumber\\
    &~~~~~~~~~~~~~~~~~~~~~~~~+ 128 \kappa \sqrt{\log(\Lambda^8)} + \frac{16 \sqrt{2E \log(\Lambda^8)}}{\sqrt{\mu}} \bigg) \Bigg)^2 - E.\label{eq:budgetSchedule}
\end{align}
\end{corollary}

We will proceed in bounding the final iterate of SGD as follows:
\begin{itemize}
    \item One main difficulty in analyzing the final iterate of SGD in our setting
    is our relaxed assumption on the norm of the gradient -- namely, we assume that the
    norm of the gradient is bounded \textit{only} at the optimal $\bb^*.$ We thus will
    rely on Lemmas \ref{lem:samplePathLem2} and \ref{lem:centeredLem2} to proceed with 
    our analysis.
    \item In Lemmas \ref{lem:origRecursion} and \ref{lem:rolledOutRecursion}, we will derive a bound on the distance from
    the optimal solution which takes a form similar to that in the expected value analysis
    of \cite{sgdHogwild, bottou2018optimization}.
    \item Afterwards, we will define a sequence of random variables $r_t$ and $V_t$, in order to prove a
    high-probability result for $d_t^2 > r_t$ in Lemma \ref{lem:highProb}.
    \item Given this high probability result, it is then sufficient to obtain an almost sure
    bound on $r_t.$ We will proceed with bounding this quantity in several stages:
    \begin{itemize}
        \item First, we obtain a useful bound on $r_t$ in Lemma \ref{lem:rtBound} which normalizes the global diameter term $D$ by a term which is polynomial in our tunable parameter $\Lambda$. Note that this step is crucial to our analysis, as $D$ can potentially grow \textit{polynomially} in the number of SGD steps $T$ under our assumptions, as we note in Remark \ref{remark:unifDiamBound}.
\item Given this bound, we are left only to bound the $V_t$ term. We first obtain a crude bound on this term in Lemma \ref{lem:VtBound}, which would allow us to achieve a $\tilde{O}(1/\sqrt{t})$ converge guarantee. We then refine this bound in Corollary \ref{cor:VtBoundTighter}, which allows us to give a convergence guarantee of $\tilde{O}(K(\gamma)/t^{\gamma})$ for any $\gamma < 1$ and for some constant $K(\gamma)$. We discuss how this refinement affects constant and $\log \Lambda$ factors in our convergence guarantee in Remark \ref{remark:C-scaling}.
        \item Finally, we collect our results to obtain our final bound on $r_{t+1}$ in
        Corollary \ref{cor:rtCollected}. 
    \end{itemize}
    \item With a bound on $r_{t+1}$ and a high probability guarantee of $d_{t+1}$ exceeding 
    $r_{t+1}$, we can finally obtain our high probability guarantee on error the final SGD
    iterate in Theorem \ref{thm:SGD_improved}.
\end{itemize}

Since quite a lot of notation will be introduced in this section, we provide a summary
of parameters used here:
\begin{center}
\begin{tabular}{@{} |p{2cm}|p{3.5cm}|p{7cm}| @{}}
\hline
Parameter & Value & Description\\
\hline
$g(\bb_t; z_t),~ g_t$ & $\nabla f(\bb_t; z_t)$ & Interchangeable notation for stochastic gradient\\
$\kappa$ & $\frac{\beta}{\mu}$ & The condition number\\
$d_{t}$ & $\| \bb_t - \bb^* \|_2^2$ & The distance of the $t$th iterate of SGD\\
$\eta_t$ & $\frac{2}{\mu(t + E)}$ & The step size of SGD\\
$E$ & $2048 \kappa^2 \log \Lambda^4$ & \\
$T$ & & The number of SGD iterations\\
$\Lambda$ & $\geq T + 1$ & Tunable parameter to control high probability bound\\
$M_t$ & $\<\nabla F(\bb_t) - g_t, \bb_t - \bb^*\>$ & \\
$\varrho_t$ & $2 d_t \sqrt{8 \beta^2 d_t^2 + 2 \GG}$ & Upper bound on the martingale difference sequence\\
$D$ & $\sup_{t=0,\ldots,T} d_t$ & The uniform diameter bound (discussed in Remark \ref{remark:unifDiamBound})\\
\hline
\end{tabular}
\end{center}

We begin by noting that crucial to our analysis is deriving bounds on our stochastic gradient, since we assume the norm of the
stochastic gradient is bounded \textit{only} at the origin. The following results are the versions of Lemma 2 from \cite{sgdHogwild}
restated as almost sure bounds.
\begin{lemma}[Sample path version of Lemma 2 from \cite{sgdHogwild}]\label{lem:samplePathLem2}
Under Assumptions \ref{assump1} and \ref{assumGrad}, the following bound on the norm of the stochastic gradient holds almost surely.
\begin{align}
    \| g(\bb_t, Z_t) \|^2 \leq 4 \beta \kappa (F(\bb_{t}) - F(\bb^*)) + 2 \GG
\end{align}
\end{lemma}
\begin{proof}
As in \cite{sgdHogwild}, we note that since
\begin{align}
    \| a - b \|^2 \geq \frac{1}{2} \| a \|^2 - \| b \|^2,
\end{align}
we may obtain the following bound:
\begin{align*}
    \frac{1}{2} \| \nabla f(\bb_t; z) \|^2 - \| \nabla f(\bb^*; z) \|^2 &\leq \| \nabla f(\bb_t; z) - \nabla f(\bb^*; z) \|^2\\
    &\leq \beta^2 \| \bb_t - \bb^* \|^2 && \text{by $\beta$-smoothness of $f$}\\
    &\leq \frac{2 \beta^2}{\mu} (F(\bb_t) - F(\bb^*)) && \text{by $\mu$-strong convexity of $F$}
\end{align*}
Rearranging, we have that
\begin{align}
    \| \nabla f(\bb_t; z) \|^2 \leq 4 \beta \kappa (F(\bb_t) - F(\bb^*)) + 2 \GG,
\end{align}
as desired.
\end{proof}

\begin{lemma}[Centered sample path version of Lemma 2 from \cite{sgdHogwild}]\label{lem:centeredLem2}
Under Assumptions \ref{assump1} and \ref{assumGrad}, for any random realization of $z$,
the following bound holds almost surely:
\begin{align}
    \| \nabla f(\bb_t; z) - \nabla F(\bb_t) \|^2 \leq 8 \beta^2 \|\bb_t - \bb^* \|^2 + 2 \GG
\end{align}
\end{lemma}
\begin{proof}
The proof proceeds similarly to Lemma \ref{lem:samplePathLem2}, replacing the stochastic gradient with the mean-centered version to obtain:
\begin{align*}
    &\frac{1}{2} \| \nabla f(\bb_t; z) - \EE[\nabla f(\bb_t; z)]\|^2
    - \| \nabla f(\bb^*; z) - \EE[\nabla f(\bb^*; z)]\|^2\\
    &\leq \| \nabla f(\bb_t; z) - \nabla f(\bb^*; z) - \EE[\nabla f(\bb_t; z)] + \EE[\nabla f(\bb^*; z)] \|^2\\
    &\leq 2(\| \nabla f(\bb_t; z) - \nabla f(\bb^*; z) \|^2 + \| \EE[\nabla f(\bb_t; z)] - \EE[\nabla f(\bb^*; z)] \|^2)\\
    &\leq 2(\| \nabla f(\bb_t; z) - \nabla f(\bb^*; z) \|^2 + \EE[\| \nabla f(\bb_t; z) - \nabla f(\bb^*; z)] \|^2])\\
    &\leq 4\beta^2\| \bb_t - \bb^*\|^2
\end{align*}
Now, rearranging terms, and recalling that $\EE[\nabla f(\bb^*; z)] = \nabla F(\bb^*) = 0$, we have
\begin{align*}
    \| \nabla f(\bb_t; z) - \nabla F(\bb_t; z) \|^2 &= \| \nabla f(\bb_t; z) - \EE[\nabla f(\bb_t; z)] \|^2\\
    &\leq 8 \beta^2 \|\bb_t - \bb^* \|^2 + 2 \|\nabla f(\bb^*; z) \|^2\\
    &\leq 8 \beta^2 \|\bb_t - \bb^* \|^2 + 2 \GG
\end{align*}
as desired. 
\end{proof}

Given these bounds on the norm of the stochastic gradient, we are now prepared to 
begin deriving high probability bounds on the optimization error of the final iterate.

\begin{lemma}\label{lem:origRecursion}
Suppose $F$ and $f$ satisfy Assumptions \ref{assump1} and \ref{assumGrad}. Consider the stochastic gradient iteration $\bb_{t+1} = \bb_t - \eta_t \nabla f(\bb_t; z_t)$, where $z$ is sampled randomly from a distribution $p(z).$ Let $\bb^* = \arg\min_{\bb} F(\bb).$ Let $M_t = \< \nabla F(\bb_t) - g(\bb_t, Z_t), \bb_t - \bb^* \>,$ where $g(\bb, z) = \nabla f(\bb, z)$ . Additionally, let us adopt the notation $d_t = \| \bb_t - \bb^* \|_2$. Then the iterates satisfy the following inequality:
\begin{align}
    d_{t+1}^2 \leq (1 - \mu \eta_t) d_t^2 + 2 \GG \eta_t^2 + 2 \eta_t M_t
\end{align}
as long as $0 < \eta_t \leq \frac{1}{2 \beta \kappa}$, where $\kappa = \frac{\beta}{\mu}.$
\end{lemma}
\begin{proof}
The proof crucially relies on techniques employed in \cite{sgdHogwild}, and in particular, on Lemma \ref{lem:samplePathLem2},
We now apply this result to bound $d_{t+1}:$
\begin{align*}
    \| \bb_{t+1} - \bb^* \|^2 &= \| \bb_t - \eta_t g(\bb_t; z_t) - \bb^* \|^2 && \text{by definition of SGD}\\
    &= \| \bb_t - \bb^* \|^2 + \eta_t^2 \| g(\bb_t; z_t) \|^2\\
    &~~~~ - 2 \eta_t \< g(\bb_t; z_t), \bb_t - \bb^* \>\\
    &\leq \| \bb_t - \bb^* \|^2 + 2 \eta_t^2 (\GG + 2 \beta \kappa (F(\bb_t) - F(\bb^*)))\\
    &~~~~ - 2 \eta_t (\<\nabla F(\bb_t), \bb_t - \bb^* \>\\
    &~~~~~~~~~~~~ + \< g_t - \nabla F(\bb_t), \bb_t - \bb^* \>) && \text{by Lemma \ref{lem:samplePathLem2}}\\
    &\leq \| \bb_t - \bb^* \|^2 + 2 \eta_t^2 (\GG + 2 \beta \kappa (F(\bb_t) - F(\bb^*)))\\
    &~~~~ - 2 \eta_t (F(\bb_t) - F(\bb^*) + \frac{\mu}{2}\| \bb_t - \bb^* \|^2\\
    &~~~~~~~~~~~~ + \< g_t - \nabla F(\bb_t), \bb_t - \bb^* \>) && \text{by $\mu$-s.c. of $F$}\\
    &= (1 - \mu \eta_t)\| \bb_t - \bb^* \|^2\\
    &~~~~{\color{blue}-} 2 \eta_t (1 - 2 \beta\kappa \eta_t)(F(\bb_t) - F(\bb^*))\\
    &~~~~ - 2 \eta_t \< g_t - \nabla F(\bb_t), \bb_t - \bb^* \> + 2 \GG \eta_t^2\\
    &\leq (1 - \mu \eta_t) d_t^2 + 2 \GG \eta_t^2 + 2 \eta_t M_t && \text{assuming $\eta_t \leq \frac{1}{2\beta\kappa}$}
\end{align*}
which is the desired result.
\end{proof}

Now given this recursion, we may derive a bound on $d_{t+1}$ in a similar form as expected value results from Theorem 2 from \cite{sgdHogwild} and Theorem 4.7 in \cite{bottou2018optimization}. Namely,
\begin{lemma}\label{lem:rolledOutRecursion}
Using the same assumptions and notation as in Lemma \ref{lem:origRecursion}, by choosing
$\eta_t = \frac{2}{\mu(t + E)}$, where $E \geq 4 \kappa^2$
we have the following bound on the distance from the optimum:
\begin{align*}
    d_{t}^2 &\leq \frac{G(d_0^2, \GG)}{t + E} + \sum_{i=0}^{t-1} c(i, t-1) M_i\\
    &\leq \frac{G(d_0^2, \GG)}{t + E} + \frac{4}{\mu(t + E)} \sum_{i=0}^t M_i
\end{align*}
where
\begin{align*}
    G(d_0^2, \GG) = \max\{E d_0^2, \frac{8\GG}{\mu^2}\}, \text{ and } c(i,t) = 2 \eta_i \prod_{j = i+1}^t (1 - \mu \eta_j)
\end{align*}
\end{lemma}
\begin{proof}
We first note that our choice of $\eta_t$ does indeed satisfy $\eta_t \leq \frac{1}{2 \beta \kappa},$ so we may apply Lemma \ref{lem:origRecursion}.

As in the aforementioned theorems, our proof will proceed inductively.

Note that the base case of $t=0$ holds trivially by construction. Now let us suppose the bound holds for some $l < t.$ Then, using the
recursion derived in Lemma \ref{lem:origRecursion}, we have that
\begin{align*}
    d_{l+1}^2 &\leq (1 - \mu \eta_l) d_l^2 + 2 \GG \eta_t^2 + 2 \eta_t M_t\\
    &\leq (1 - \mu \eta_l) \left(\frac{G(d_0^2, \GG)}{l + E} + \sum_{i=0}^{l-1} c(i,l-1)M_i \right) + 2 \GG \eta_l^2 + 2 \eta_l M_l\\
    &= (1 - \mu \eta_l) \frac{G(d_0^2, \GG)}{l + E} + 2 \GG \eta_l^2 + \sum_{i=0}^l c(i,l)M_i\\
    &= G(d_0^2, \GG)\frac{l + E - 2}{(l + E)^2} + \frac{8 \GG}{\mu^2 (l + E)^2} + \sum_{i=0}^l c(i,l)M_i\\
    &= G(d_0^2, \GG)\frac{l + E - 1}{(l + E)^2} - \frac{G(d_0^2, \GG)}{(l + E)^2} + \frac{8 \GG}{\mu^2 (l + E)^2} + \sum_{i=0}^l c(i,l)M_i\\
\end{align*}
Now note that, by definition of $G(d_0^2, \GG)$, we have that
\begin{align}
    - \frac{G(d_0^2, \GG)}{(l + E)^2} + \frac{8 \GG}{\mu^2(l + E)^2} \leq 0
\end{align}
Therefore, we find that
\begin{align*}
d_{l+1}^2 &\leq G(d_0^2, \GG) \frac{l + E - 1}{(l + E)^2} + \sum_{i=0}^l c(i,l)M_i\\
&= G(d_0^2, \GG) \frac{(l + E)^2 - 1}{(l + E)^2} \frac{1}{t + E + 1} + \sum_{i=0}^l c(i,l)M_i\\
&\leq \frac{G(d_0^2, \GG)}{(l + 1) + E} + \sum_{i=0}^l c(i,l)M_i\\
\end{align*}
Thus, the result holds for all $t.$

We now note that $c(i,t) \leq \frac{4}{\mu(t + E)}.$ Observe that
\begin{align*}
    c(i,t) &= 2 \eta_i \prod_{j = i+1}^t (1 - \mu\eta_j)\\
    &= \frac{4}{\mu(i + E)} \prod_{j = i+1}^t \frac{j + E - 2}{j + E}\\
    &= \frac{4}{\mu(i + E)} \frac{i + E - 1}{t + E}\\
    &\leq \frac{4}{\mu(t + E)}
\end{align*}
\end{proof}

Now, in order to obtain a high probability bound on the final iterate of SGD, we need to obtain a concentration result for $\sum_{i=0}^t M_i.$
We note that, from Lemma \ref{lem:centeredLem2}, we obtain an upper bound on the magnitude of $M_i:$
\begin{align*}
    | M_t | &\leq \| g(\bb_t; z_t) - \nabla F(\bb_t) \| \| \bb_t - \bb^* \|\\
    &\leq \sqrt{8 \beta^2 d_t^2 + 2 \GG} d_t.
\end{align*}

We consider the usual filtration $\mathcal{F}_t$ that is generated by $\{z_i\}_{i \leq t}$ and $\mathbf{w}_0$. Just for completeness of notation we set $z_0=0$ (no gradient at step $0$).

By this construction,  we observe that $M_t  $ is a martingale difference sequence with respect to the filtration $\mathcal{F}_t$. In other words, $S_t=\sum_{s=1}^t M_s $ is a martingale.

\begin{lemma}\label{lem:condexp}
  $\EE[M_t \mid \mathcal{F}_{t-1}]=0,~\forall t>0$.
\end{lemma}
\begin{proof}
 Given the filtration, $\mathcal{F}_{t-1}$, $\mathbf{w}_0,z_1 \ldots z_{t-1}$ is fixed. This implies that $\mathbf{w}_t $ is fixed. However, conditioned on $\{z_i\}_{i <t}$, $z_t$ is randomly sampled from $p(z)$.
 Therefore,
  $\EE[ g(\mathbf{w}_t,z_t) - \nabla F(\mathbf{w}_t) \mid \mathcal{F}_{t-1}] = \EE_{z_t \mid \mathcal{F}_{t-1} }[ g(\mathbf{w}_t,z_t) - \nabla F(\mathbf{w}_t) \mid \mathbf{w}_t] =\EE_{z_t \sim p(z) }[ g(\mathbf{w}_t,z_t) - \nabla F(\mathbf{w}_t) \mid \mathbf{w}_t] =0 $. Hence,
  $\EE[M_t \mid \mathcal{F}_{t-1} ]=0$ 
\end{proof}

Recall that, $M_s$ is uniformly upper bounded by 
$\varrho_t =   d_t \sqrt{8 \beta^2 d_t^2 + 2 \GG}$.
Thus, we have that 
$\varrho_t^2 \leq d_t^2 (8 \beta^2 d_t^2 + 2 \GG).$

Let $D = \sup_{0 \leq t \leq T} d_t$. Then, $|M_t| \leq d_t \sqrt{8 \beta^2 d_t^2 + 2 \GG} \leq \tilde{C}(D^2, D\sqrt{\GG})=D \sqrt{8\beta^2 D^2 + 2\GG}$.

In order to obtain a high probability bound on the final SGD iterate, we will introduce
the following sequence of random variables and events, and additionally constants $c'(t)$ to be decided later. 
\begin{enumerate}
    \item \textbf{Initialization at $t=0$:} Let $V_0 =  \frac{8 d_0^2 (4 \beta^2 d_0^2 + \GG)}{1 + E},$ $r_0 = d_0^2$, and take $\mathcal{A}_0$ to be an event that is true with probability 1. Let $M_0=0$. $\mathrm{Pr}(\mathcal{E}_0)$=1, $\delta_0=0$. \item $r_{t} = \frac{G}{t +  E} + \frac{4}{\mu(t +  E)} (t-1)\delta_{t-1} \tilde{C}(D^2, D\sqrt{\GG}) + \frac{4}{\mu}\sqrt{\frac{2 \log (\Lambda^8/c'(t))}{t+E}}\sqrt{V_{t-1}}$
    
\item $V_{t} = \frac{1}{t+E+1}\sum_{i=0}^t 8 r_i (4 \beta^2 r_i + \GG)$.
    \item Event $\mathcal{A}_{t}$ is all sample paths satisfying the condition: $d_{t}^2 \leq r_{t}$.
    \item Let $\mathcal{E}_t = \bigcap_{i \leq t} \mathcal{A}_i$. Further, let $\mathrm{Pr}(\mathcal{E}^c_t)/\mathrm{Pr}(\mathcal{E}_t) = \delta_t.$
\end{enumerate}

We now state a conditional form of the classic Azuma-Hoeffding inequality that has been tailored to our setting, and provide a proof for completeness.

\begin{lemma}[Azuma-Hoeffding with conditioning] \label{lem:ah-cond}
Let $S_n=f(z_1 \ldots z_n)$ be a martingale sequence with respect to the filtration ${\cal F}_n$ generated by $z_1\ldots z_n$. Let $\psi_n=S_n-S_{n-1}$. Suppose $\lvert \psi_n \rvert \leq c_{n} (z_1 \ldots z_{n-1}) $ almost surely. Suppose $E[\psi_n \mid {\cal F}_{n-1}]=0$.

Let ${\cal A}_{n-1}$ 
 be the event that $c_n \leq d_{n},$ where ${\cal A}_{n-1}$ is defined on the filtration ${\cal F}_{n-1},$ and $d_n$ is a constant dependent only on the index $n$. Define ${\cal E}_n = \bigcap_{i \leq n} {\cal A}_i$. Further suppose that that $\exists \bar{R}$ large enough such that $\lvert \psi_n \rvert \leq \bar{R}$ almost surely. 
 Finally let
 $\mathrm{Pr}(\mathcal{E}^c_n)/\mathrm{Pr}(\mathcal{E}_n) = \delta_n.$
Then,
 \begin{align}
   \mathrm{Pr}\left( S_n \geq \gamma +  n \delta_n \bar{R} \ \lvert \ {\cal E}_n \right) \leq \exp \left( -\frac{\gamma^2}{2 \sum_{i=1}^n d_i^2 } \right)
 \end{align}
\end{lemma}
\begin{proof}
We first observe that $\EE[\psi_i \mid {\cal F}_{i-1} ] =0$. Therefore, for $i \leq n$ we have:
\begin{align}\label{bound:bad}
    \lvert \EE[\psi_i \mid \mathcal{E}_n, {\cal F}_{i-1}] | &= \frac{\mathrm{Pr}(\mathcal{E}_n^c)}{\mathrm{Pr}( \mathcal{E}_n)} | \EE[\psi_i \mid \mathcal{E}_n^c, {\cal F}_{i-1}] | \nonumber \\
    &\leq \frac{\mathrm{Pr}(\mathcal{E}_n^c)}{\mathrm{Pr}( \mathcal{E}_n)} \bar{R}  \nonumber \\
    &\leq \delta_n \bar{R}
\end{align}

Consider the sequence $S'_i= S_i - \sum_{j=1}^i \EE[\psi_j \mid {\cal F}_{j-1},  \mathcal{E}_n ]$ for $i \leq n$.

\begin{align} \label{azumaravel}
\mathrm{Pr}(S'_i \geq \gamma \mid \mathcal{E}_n ) & \leq e^{-\theta \gamma} \EE [ e^{\theta S'_i} \mid \mathcal{E}_n ] \nonumber \\
\hfill &= e^{-\theta \gamma} \EE[ \EE [e^{\theta S'_i} \mid \mathcal{E}_n,{\cal F}_{i-1}   ] \mid \mathcal{E}_n] \nonumber \\
\hfill &= e^{-\theta \gamma} \EE[ e^{\theta S'_{i-1}} \EE [e^{\theta (\psi_i - \EE[\psi_i \mid {\cal F}_{i-1},  \mathcal{E}_n ] )} \mid \mathcal{E}_n,{\cal F}_{i-1}   ]  \mid \mathcal{E}_n ] \nonumber \\
\end{align}

Observe that $ \EE[ \psi_i - \EE[\psi_i \mid {\cal F}_{i-1},  \mathcal{E}_n ] \mid {\cal F}_{i-1},  \mathcal{E}_n ] =0$. i.e. $\psi_i - \EE[\psi_i \mid {\cal F}_{i-1},  \mathcal{E}_n ] $ is a mean $0$ random variable with respect to the conditioning events ${\cal F}_{i-1},{\mathcal E}_{n}$.

Further, for any sample path where ${\mathcal E}_{n}$ holds, we almost surely have $\lvert \psi_i - \EE[\psi_i \mid {\cal F}_{i-1},  \mathcal{E}_n ] \rvert \leq 2c_i (z_1,z_2 \ldots z_{i-1}) \leq 2 d_i.$

Therefore, $\EE [e^{\theta (\psi_i - \EE[\psi_i \mid {\cal F}_{i-1},  \mathcal{E}_n ] )} \mid \mathcal{E}_n,{\cal F}_{i-1}   ] \leq e^{4\theta d_i^2/2}$

Therefore, (\ref{azumaravel}) yields the following:
\begin{align}
\mathrm{Pr}(S'_i \geq \gamma \mid \mathcal{E}_n ) & \leq
    e^{-\theta \gamma} \EE[ e^{\theta S'_{i-1}} \mid \mathcal{E}_{n} ] [e^{\frac{4\theta d_i^2}{2}}] \nonumber \\
    \hfill &=  e^{-\theta \gamma} e^{\theta \sum_{j=1}^i 4d_j^2/2}
\end{align}
Let $\theta = \frac{\gamma}{\sum_{i=1}^n 4 d_i^2}$. Then, we have for $i=n$:
\begin{align}
 & \mathrm{Pr}\left( S_n \geq \gamma + \sum_{i=1}^n \EE[\psi_i \mid {\cal F}_{i-1},  \mathcal{E}_i ] ~\bigg|~ {\cal E}_n \right) \leq \exp \left( -\frac{\gamma^2}{8 \sum_{i=1}^n d_i^2 } \right) \nonumber \\
\overset{a}{\Rightarrow} &  \mathrm{Pr}\left( S_n \geq \gamma +  n\delta_n \bar{R} \ \lvert \ {\cal E}_n \right) \leq \exp \left( -\frac{\gamma^2}{8 \sum_{i=1}^n d_i^2 } \right)   
\end{align}
(a) - This is obtained by substituting the almost sure bound (\ref{bound:bad}) for all $i \leq n$.

\end{proof}

Using our iterative construction and the conditional Azuma-Hoeffding inequality,
we obtain the following high probability bound:
\begin{lemma}\label{lem:highProb}
Under the construction specified above, we have the following:
\begin{align}
    \mathrm{Pr}(d_{t+1}^2 > r_{t+1} \mid \mathcal{E}_t) &\leq \frac{c'(t+1)}{\Lambda^8}
\end{align}
When $c'(i)=1$, we have:
\begin{align}
    \mathrm{Pr}(\mathcal{E}_{t+1}^c) \leq \frac{t+1}{\Lambda^8}
\end{align}
\end{lemma}
\begin{proof}
By the conditional Azuma-Hoeffding Inequality (Lemma~\ref{lem:ah-cond}), we have the following chain:
\begin{align*}
    \mathrm{Pr} ({\cal A}^c_{t+1} \lvert {\cal A}_i,~ i \leq t) &{= \mathrm{Pr} (d_{t+1}^2 > r_{t+1} \mid {\cal A}_i,~ i\leq t)} \nonumber\\
    &\leq \mathrm{Pr} \Bigg(\frac{4}{\mu(t + 1 + E)}\sum_{i=1}^t (M_i - \delta_{t} \tilde{C}(D^2, D\sqrt{\GG})) >\nonumber\\
    &~~~~~~~~~~~~\frac{4}{\mu(t + 1 + E)} \sqrt{\sum_{i = 0}^t \varrho_i^2} \sqrt{2 \log\left(\frac{\Lambda^8}{c'(t+1)}\right)} ~\bigg|~ {\cal A}_i,~ i \leq t\Bigg) \nonumber\\
    &{ \overset{a}{\leq} \exp\left( - \frac{(2\log(\frac{\Lambda^8}{c'(t+1)}))\sum_{i=0}^t  \varrho_i^2}{2\sum_{i=0}^t \varrho_i^2} \right)} \nonumber\\
    &= \frac{c'(t+1)}{\Lambda^8} 
\end{align*}

(a)- We set $\psi_i$ in Lemma \ref{lem:ah-cond} to be the variables $M_i$, filtrations ${\cal F}_t$ to be that generated by $z_t \sim p(z)$ (and $\mathbf{w}_0$) in the stochastic gradient descent steps. $c_t$ (in Lemma \ref{lem:ah-cond}) set to $\varrho_t$, $d_t$ (in Lemma \ref{lem:ah-cond}) is set to $r_t$ , $\bar{R}$ (in Lemma \ref{lem:ah-cond}) is set to $\tilde{C}(D^2, D\sqrt{\GG})$ and $\delta_t$ (in Lemma \ref{lem:ah-cond}) is set to $\mathrm{Pr}(\mathcal{E}^c_t)/\mathrm{Pr}(\mathcal{E}_t)$. Now, if we apply Lemma \ref{lem:ah-cond} to the sequence $M_i$ with the deviation $\gamma$ set to $\sqrt{\sum_{i = 0}^t \varrho_i^2} \sqrt{2 \log\left(\frac{\Lambda^8}{c'(t+1)}\right)}$, we obtain the inequality.

\begin{align}
 \mathrm{Pr}(\mathcal{E}_{t+1}^c) &\leq  \sum_{i=1}^{t+1} \mathrm{Pr}( \min \{j: d_j^2 > r_j\}=i) \nonumber\\
 & \leq \sum_{i=1}^{t+1} \mathrm{Pr}(\mathcal{A}_{i}^c \mid {\cal A}_j,~j < i) = \sum_{i=1}^{t+1} \frac{c'(i)}{\Lambda^8}
\end{align}
Choosing $c'(i)=1$, we thus obtain our desired result.

\end{proof}

From Lemma \ref{lem:highProb}, we have a high probability bound on the event that $d_t^2 > r_t.$
In order to translate this to a meaningful SGD convergence result, we will have to substitute for $\delta_t$. 
We thus upper bound $r_t$ as follows:

\begin{lemma}\label{lem:rtBound}
Under the above construction, where $c'(i)$ is chosen to be $1$, we have the following almost sure upper bound on $r_{t},~ \forall ~t \leq \Lambda$
\begin{align}
    r_{t} &\leq \frac{G(d_0^2, \GG)}{t + E } + \frac{8 t\tilde{C}(D^2, D\sqrt{\GG})}{\mu(t  + E)\Lambda^7} + \frac{4 \sqrt{2 \log(\Lambda^8)}\sqrt{V_{t-1}}}{\mu\sqrt{t + E }}
\end{align}
where $\tilde{C}(D^2, D\sqrt{\GG}) = D \sqrt{8\beta^2 D^2 + 2\GG}$, and $D$ is taken to be
a uniform diameter bound\footnote{See Remark \ref{remark:unifDiamBound} for a discussion on our reasoning for using a global diameter bound here.}.
\end{lemma}
\begin{proof}
From Lemma \ref{lem:highProb}, we have:
  $\delta_t = \frac{\mathrm{Pr}(\mathcal{E}_t^c)}{\mathrm{Pr}(\mathcal{E}_t)} \leq \frac{t}{\Lambda^8- t} \leq \frac{2}{\Lambda^7}$. 
Here, we assume that $\Lambda >2$. Substituting in the expression for $r_t$, we have the result.  
\end{proof}

Given this bound from Lemma \ref{lem:rtBound}, we now must construct an upper bound on $V_t.$
We will proceed in two steps, first deriving a crude bound on $V_t$, and then by iteratively refining
this bound. We now derive the crude bound.
\begin{lemma}\label{lem:VtBound}
The following bound on $V_t$ holds almost surely:
\begin{align}
    V_t \leq \check{C} \log \Lambda^8
\end{align}
assuming that we choose
\begin{align*}
    E &\geq 128 \beta^2 c_2^2 \log \Lambda^8\\
    \check{C} &\geq \max\left\{ \frac{V_0}{\log\Lambda^8}, (8\GG c_2 + \min\{2/E, 1\})^2, \left( \frac{64 \beta^2 c_1^2}{(1+E)\log \Lambda^8} + \frac{8 \GG c_1}{\log \Lambda^8}\right)^2 \right\}\\
    c_1 &= G(d_0^2, \GG) + \frac{8\tilde{C}(D^2, D\sqrt{\GG})}{\mu \Lambda^6}  \\ 
c_2 &= \frac{4\sqrt{2}}{\mu}\\
\Lambda &\geq t + 1
\end{align*} 
\end{lemma}
\begin{proof}
We will prove the claim inductively.

We note that the base case when $t=0$ holds by construction, assuming that $\check{C} \geq \frac{V_0}{\log \Lambda^4}$.

Now let us suppose that our claim holds until some $t.$ Then by applying the bound on $r_t$ derived
in Lemma \ref{lem:rtBound}, we have the following bound:
\begin{align*}
r_{t+1} &\leq \frac{G(d_0^2, \GG)}{t + 1 + E } + \frac{8 (t+1)\tilde{C}(D^2, D\sqrt{\GG})}{\mu(t + 1  + E)\Lambda^7} + \frac{4\sqrt{2 \check{C}} \log \Lambda^8}{\mu \sqrt{t + E + 1}}\\
    &\leq \frac{c_1}{t + E + 1} + \frac{c_2 \sqrt{\check{C}} \log \Lambda^8}{\sqrt{t + E + 1}},
\end{align*}
where $c_1 = G(d_0^2, \GG) + \frac{8\tilde{C}(D^2, D\sqrt{\GG})}{\mu \Lambda^6}$
and $c_2 = \frac{4\sqrt{2}}{\mu}$.
Plugging in this bound to our definition of $V_{t+1},$ we obtain:
\begin{align*}
    V_{t+1} &= \frac{t + 1 + E}{t + 2 + E}V_t + 32 \beta^2 r_{t+1}^2 + 8 \GG r_{t+1}\\
    &\leq \frac{\check{C} \log\Lambda^8 }{t + E + 2} \left( (t+E+1) + \frac{64\beta^2c_2^2\log\Lambda^8}{t+E+1} + \frac{8 \GG c_2}{\sqrt{\check{C} (t+E+1)}} \right)\\
    &~~~~+ \frac{1}{(t+E+1)(t+E+2)}\left( \frac{64 \beta^2 c_1^2}{(t+E+1)} + 8 \GG c_1 \right)\\
    &\overset{\mathrm{shown~below}}{\leq} \check{C} \log \Lambda^8
\end{align*}
Rearranging, we find that we equivalently need:
\begin{align*}
\frac{64\beta^2 c_1^2}{(t+E+1)\log\Lambda^8} + \frac{8 \GG c_1}{\log\Lambda^8} &\leq \sqrt{\check{C}}(\sqrt{\check{C}}(t+E+1 - 64\beta^2 c_2^2 \log\Lambda^8)\\
&~~~~~~~~~~~~~~~~ - 8\GG c_2 \sqrt{t+E+1}).
\end{align*}
Now, setting $E = 2 * 64 \beta^2 c_2^2 \log\Lambda^8$, we find that
a sufficient condition to complete our induction hypothesis is:
\begin{align}
\frac{64\beta^2 c_1^2}{(t+E+1)\log\Lambda^8} + \frac{8 \GG c_1}{\log\Lambda^8} &\leq \sqrt{\check{C}}((\sqrt{\check{C}} - 4\GG c_2)(t+1) + (\sqrt{\check{C}} - 8\GG c_2)E/2).\label{eq:vtIndCond}
\end{align}
Now, observe that by choosing
\begin{align}
    \check{C} \geq \max \left\{ (8\GG c_2 + \min\{2/E, 1\})^2, \left( \frac{64 \beta^2 c_1^2}{(1+E)\log \Lambda^8} + \frac{8 \GG c_1}{\log \Lambda^8}\right)^2 \right\}
\end{align}
the sufficient condition (\ref{eq:vtIndCond}) is satisfied. Hence, our claim holds for all $t$.
\end{proof}

Now given this crude upper bound, we may repeatedly apply Lemma 8 from \cite{sgdHogwild} in order to obtain the following result:
\begin{corollary}[of Lemma \ref{lem:VtBound} + Lemma 8 in \cite{sgdHogwild}]\label{cor:VtBoundTighter}
    After $k\geq 0$ applications of Lemma 8 from \cite{sgdHogwild}, 
    under the same assumptions as in Lemma \ref{lem:VtBound}, we have the following bound on $V_t:$
    \begin{align}
        V_{t} \leq \frac{\hat{C}(k)}{(t + E + 1)^{\alpha_{k}}}
    \end{align}
    where
    \begin{align*}
        \hat{C}(k+1) &= 2^{k+1} C(k+1) + V_0 \frac{1 + E}{(2+E)^{1-\alpha_{k+1}}}\\
        C(k+1) &= \frac{64\beta^2 c_1^2}{(E+1)^{2-\alpha_{k+1}}} + \frac{64 \hat{C}(k) c_2^2}{\mu^2 (E+1)^{\alpha_{k+1}}}\\
        &~~~~+ \frac{8\GG c_1}{(E+1)^{1-\alpha_{k+1}}} + 8\frac{\GG}{\mu} \sqrt{\hat{C}(k)} c_2\\
        \alpha_{k+1} &= \sum_{i=1}^{k+1} \frac{1}{2^i}\\
        \hat{C}(0) &= \check{C} \log \Lambda^8\\
        \alpha_0 &= 0
    \end{align*}
    where $E, \check{C}, c_1, c_2$ are defined as in Lemma \ref{lem:VtBound}.
\end{corollary}

\begin{proof}
We will construct this bound inductively. We begin by noting that, when $k=0$, the bound holds 
by Lemma \ref{lem:VtBound}.
Now let us assume the bound holds until some $k.$ Observe, then, that, by plugging into the bound
in Lemma \ref{lem:rtBound}, we may write
\begin{align}
    V_{t+1} \leq \beta_t V_t + \gamma_t
\end{align}
where 
\begin{align*}
    \beta_t &= \frac{t+1+E}{t+2+E}\\
    \gamma_t &= \frac{C(k+1)}{(t+E+1)^{\alpha_{k+1}}(t+E+2)}\\
    C(k+1) &= \frac{64\beta^2 c_1^2}{(E+1)^{2-\alpha_{k+1}}} + \frac{64 \hat{C}(k)c_2^2}{\mu^2 (E+1)^{\alpha_{k+1}}} + \frac{8\GG c_1}{(E+1)^{1-\alpha_{k+1}}} + 8\frac{\GG}{\mu} \sqrt{\hat{C}(k)}c_2\\
\end{align*}
Now, we may apply Lemma 8 in \cite{sgdHogwild} to obtain:
\begin{align*}
    V_{t+1} &\leq \sum_{i=0}^{t} \left( \prod_{j=i+1}^t \beta_j \right) \gamma_i + V_0\prod_{i=0}^{t}\beta_i\\ 
    &= \sum_{i=0}^t \frac{i+2+E}{t+2+E} \frac{C(k+1)}{(i + E + 1)^{\alpha_{k+1}}(i + E + 2)} + V_0\frac{1 + E}{t + 2 + E}\\
    &\leq \frac{C(k+1)}{t+2+E} \int_{E}^{t+1+E} \frac{1}{x^{\alpha_{k+1}}} dx + V_0 \frac{1+E}{t+2 +E}\\
    &\leq \frac{C(k+1)}{(1-\alpha_k)(t+E+2)^{\alpha_{k+1}}} + V_0 \frac{1 + E}{t + E + 2}\\
    &\leq \frac{\hat{C}(k+1)}{(t + E + 2)^{\alpha_{k+1}}},
\end{align*}
where $\hat{C}(k+1) = \frac{C(k+1)}{1 - \alpha_{k+1}} + V_0 \frac{1+E}{(2+E)^{1-\alpha_{k+1}}} = 2^{k+1} C(k+1) + V_0 \frac{1+E}{(2+E)^{1-\alpha_{k+1}}}$.

Thus, our claim holds for all $k.$
\end{proof}

\begin{remark}
\label{remark:C-scaling}
Note that while $\hat{C}(k)$ in Corollary \ref{cor:VtBoundTighter} has complicated dependencies on $\GG,$ $d_0,$
$\beta$, and $\mu$, it is straightforward to argue that $\hat{C}(k) \leq \rho_k \log \Lambda^8$ where $\rho_k$ is a constant
that is independent of $\Lambda$. Indeed, note that, from Corollary \ref{cor:VtBoundTighter}, we have that
\begin{align*}
    \hat{C}(k+1) &\leq 2^{k+1} \left[ \frac{e_1}{(E+1)^{2-\alpha_{k+1}}} + \frac{e_2}{(E+1)^{1-\alpha_{k+1}}} + \frac{e_3 }{(E+1)^{\alpha_{k+1}}}\hat{C}(k) + e_4 \sqrt{\hat{C}(k)} \right]\\
    &~~~~+ \frac{e_5}{(2+E)^{1-\alpha_{k+1}}}
\end{align*}
for some $e_1,\ldots,e_5$ which are independent of $\Lambda.$ 
Note that when $k=0$, the claimed bound on $\hat{C}(0)$ holds by definition, for proper choice of $\rho_0.$ Assuming the bound holds until $k,$
we may construct a bound of the desired form by choosing $\rho_{k+1}$ as a function of the $e_i$s. Note that $E +1 \geq 1,$ and that each $e_i$ is independent of
$\Lambda$, so $\rho_{k+1}$ is also independent of $\Lambda.$ We may thus conclude that $\hat{C}(k) = O(\log \Lambda^8)$.
\end{remark}

We may collect these results to obtain:
\begin{corollary}[of Lemma \ref{lem:rtBound} + Corollary \ref{cor:VtBoundTighter}]\label{cor:rtCollected}
Under the assumptions on $E$ in Lemma \ref{lem:VtBound} and the definition of $\hat{C}(k)$ from Corollary \ref{cor:VtBoundTighter}, the following bound holds almost surely, for any $k\geq 0$,
\begin{align}
    r_{t+1} &\leq \frac{G(d_0^2, \GG)}{t + E + 1} + \frac{8 (t+1)\tilde{C}(D^2, D\sqrt{\GG})}{\mu(t + 1 + E)\Lambda^7} +  \frac{4 \sqrt{2 \log(\Lambda^8)}\sqrt{\hat{C}(k)}}{\mu(t + E + 1)^{\sum_{i=1}^{k+1} 2^{-i}}}
\end{align}
\end{corollary}

We are now prepared to state and prove our main SGD result.
\begin{proof}[Proof of Theorem \ref{thm:SGD_improved}]
The proof is an immediate consequence of Lemma \ref{lem:highProb} combined with Corollary \ref{cor:rtCollected}.
\end{proof}
 \section{Putting It Together: Tree-Search}
\label{sec:tree}

\begin{lemma}
\label{lem:union}
With probability at least $1 - \frac{1}{\Lambda^3}$, Algorithm~\ref{algo:mixnmatch} only expands nodes in the set $J :=  \cup_{h=1}^{\Lambda} J_h$, where $J_h$ is defined as follows,

\begin{align*}
J_{h} := \{ \text{nodes } (h,i) \text{ such that } G(\aa_{h,i}) - 3\nu_2\rho_2^h  \leq G(\aa^*) \}. 
\end{align*}

\end{lemma}

\begin{proof}

Let $A_t$ be the event that the  leaf node that we decide to expand at time $t$ lies in the set $J$. Also let $\mathcal{L}_t$ be the set of leaf-nodes currently exposed at time $t$. Let $B_t = \bigcup_{(h,i) \in \mathcal{L}_t} \left\{ \norm{\bb_{h,i} - \bb^*_{h,i}}_2^2 \leq \nu_1 \rho^h \right\}.$ 

Now we have the following chain,
\begin{align*}
\PP(B_t^c) &=  \PP\left( \bigcup _{l = 1}^{t/2} ( \{B_t^c\} \cap \{|\mathcal{L}_t| = l\}) \right) \\
&\leq \sum_{l = 1}^{t/2} \PP \left(\{B_t^c\} \cap \{|\mathcal{L}_t| = l\} \right) \\
&\leq \sum_{l = 1}^{t/2} \PP \left(\{B_t^c\} \right) \\
& \stackrel{(a)}{\leq} \sum_{l = 1}^{t/2} \sum_{k = 1}^{l} \frac{1}{\Lambda^7} \\
& \leq \frac{1}{\Lambda^5}. 
\end{align*}
Here, $(a)$ is due to the h.p. result in Corollary~\ref{cor:step}. 
\end{proof}

Now note that due to the structure of the algorithm an optimal node (partition containing the optimal point) at a particular height has always been evaluated prior to any time $t$, for $t \geq 2$. Now we will show that if $B_t$ is true, then $A_t$ is also true. Let the $\aa^*_{(h,i)}$ be a optimal node that is exposed at time $t$. Let $b^*_{h,i}$ be the lower confidence bound we have for that node. Therefore, given $B_t$ we have that,
\begin{align*}
    b^*_{h,i} &= F^{(te)}(\bb_{h,i}) - 2\nu_2 \rho_2^h \\
    & \leq G(\aa_{h,i}) + L \norm{\bb_{h,i} - \bb^*(\aa_{h,i})}_2   - 2\nu_2 \rho_2^h \\
    & \leq G(\aa^*)
\end{align*}

So for a node at time $t$ to be expanded the lower confidence value of that node $b_{h,i}$ should be lower than $G(\aa^*)$. Now again given $B_t$ we have that, 
\begin{align*}
    b_{h,i} &= F^{(te)}(\bb_{h,i}) - 2\nu_2 \rho_2^h \\
            &\geq G(\aa_{h,i}) - 3\nu_2 \rho_2^h. 
\end{align*}

Therefore, we have that $\PP(A_t^c) \leq \PP(B_t^c)$. Now, let $A$ be the event that over the course of the algorithm, no node outside of $J$ is every expanded. Let $T$ be the random variable denoting the total number of evaluations given our budget. We now have the following chain. 

\begin{align*}
    \PP(A^c) &= \PP\left( \bigcup _{T = 1}^{\Lambda} \left\{ \bigcup_{t=1}^{T} \{A_t^c\}\right\} \cap \{T = l\} \right) \\
    & \leq \sum_{T = 1}^{\Lambda} \PP \left( \bigcup_{t=1}^{T} \{A_t^c\}\} \right) \\
    &\leq \frac{1}{\Lambda^3}
\end{align*}

\begin{lemma}
	\label{lem:height}
	Let $h'$ be the smallest number $h$ such that $ \sum_{l = 0}^{h}
    2C(\nu_2,\rho_2) \lambda(h) \rho_2^{-d(\nu_2, \rho_2)l} > \Lambda -
    2\lambda(h+1)$. The tree in Algorithm~\ref{algo:mixnmatch} grows to a
    height of at least $h(\Lambda) = h'+1$, with probability at least $1 -
    \frac{1}{\Lambda^3}$.  Here, $\lambda(h)$ is as defined in Corollary~\ref{cor:step}. 
\end{lemma}

\begin{proof}
	We have shown that only the nodes in $J = \cup_{h}J_h$ are expanded. Also, note that by definition $\lvert J_h\rvert \leq C(\nu_2,\rho_2) \rho_2^{-d(\nu_2, \rho_2)}$. 
	
	Conditioned on the event $A$ in Lemma~\ref{lem:union}, let us consider the strategy that only expands nodes in $J$, but expands the leaf among the current leaves with the least height. This strategy yields the tree with minimum height among strategies that only expand nodes in $J$. The number of s.g.d steps incurred by this strategy till height $h'$ is given by, 
	
	\begin{align*}
	\sum_{l = 0}^{h'}2 C(\nu_2,\rho_2) \lambda(l) \rho_2^{-d(\nu_2, \rho_2)l}.
	\end{align*}
 	
 Since the above number is greater than to $\Lambda -2\lambda(h'+1)$ another set of children at height $h'+1$ is expanded and then the algorithm terminates because of the check in the while loop in step 4 of Algorithm~\ref{algo:mixnmatch}. Therefore, the resultant tree has a height of at least $h'+1$.
\end{proof}

\begin{proof}[Proof of Theorem~\ref{thm:tree}]
Given that event $A$ in Lemma~\ref{lem:union} holds, Lemma~\ref{lem:height} shows that at least one node at height $h'$ (say $(h',i)$) is expanded and one of that node's children say $\aa_{h'+1,i'}$ is returned by the algorithm. Note that $(h',i)$ is in $J_h$ and therefore $G(\aa_{h',i}) - 3\nu_2\rho_2^{h'}  \geq G(\aa^*)$. Invoking the smoothness property in Corollary~\ref{cor:alphaGeometricDecay}, we get that
\begin{equation}
    G(\aa_{h'+1,i'}) \leq G(\aa^*) + 4\nu_2\rho_2^{h'}.  
\end{equation}
\end{proof}

 \section{Scaling of \texorpdfstring{$h(\Lambda)$}{} and
\texorpdfstring{$\lambda(h)$}{}}
\label{sec:lambdaHScaling}

In this section, we discuss how to interpret the scaling of the height function 
$h(\Lambda)$ from Theorem \ref{thm:tree} and the SGD budget allocation strategy
$\lambda(h)$ from Corollary \ref{cor:step}.

Let us take $k=0$ in Theorem \ref{thm:SGD_improved}, and assume the third term
in the high probability bound is dominant: that is, for some constant $K$ large
enough, taking $C=\frac{4\sqrt{2\check{C}}K\log\Lambda^8}{\mu}$, we want to choose
$\lambda(h)$ to satisfy:
\begin{align}
    \frac{C\log\Lambda}{\sqrt{\lambda(h) + E}} \leq \nu_1\rho^{h+1}.
\end{align}
Then, solving for $\lambda(h)$, we have that
\begin{align}
    \lambda(h) &= \left(\frac{C\log\Lambda}{\nu_1\rho^{h+1}}\right)^2 - E\\
    &= \tilde{O}\left(\frac{1}{\rho^{2h}}\right)
\end{align}
Thus, outside of the constant scaling regime discussed in Corollary \ref{cor:step},
we expect SGD to take an exponential (in height) number of SGD steps in order to obtain
a solution that is of distance $\nu_1\rho^{h+1}$ from the optimal solution w.h.p. (Recall
that $\rho\in(0,1)$)

In light of this, we may discuss now how the depth of the seach tree, $h(\Lambda)$, scales
as a function of the total SGD budget $\Lambda$. We will let 
\begin{align}
    \lambda(h) = \begin{cases}
        \lambda_{const} & \text{When $h$ is in constant step size regime}\\
        \frac{C'\log^2\Lambda}{\nu_1\rho^{2 h}} & \text{Outside of this regime,
        for $C'$ chosen large enough}
    \end{cases}
\end{align}
We may thus solve for $h'$ from Theorem \ref{thm:tree} as follows. Denote
$h_{const}$ as the maximum height of the tree for which $\lambda(h)=\lambda_{const}$ for all
$h\leq h_{const}$. Then:
\begin{align*}
    \sum_{i=0}^{h(\Lambda)-1} 2C(\nu_2,\rho_2)\lambda(l)\rho^{-d(\nu_2,\rho_2)l} &=
    2C(\nu_2,\rho_2)\lambda_{const}\sum_{i=0}^{h_{const}} \rho^{-d(\nu_2,\rho_2)l} + 2
    \tilde{C}\log^2\Lambda\sum_{l=h_{const}+1}^{h(\Lambda)-1}\rho^{-(d(\nu_2,\rho_2)+2)l}\\
    &= \underbrace{2C(\nu_2,\rho_2)\lambda_{const}
    \frac{\rho^{-d(\nu_2,\rho_2)(h_{const}+1)}-1}{\rho^{-d(\nu_2,\rho_2)}-1}}_{T1}\\
    &~~~~+ \underbrace{2\tilde{C}\log^2\Lambda\frac{\rho^{-(d(\nu_2,\rho_2)+2)h(\Lambda)}-\rho^{-(d(\nu_2,\rho_2)+2)(h_{const}+2)}}{\rho^{-(d(\nu_2,\rho_2)+2)}-1}}_{T2}\\
    &\overset{want}{>} \Lambda - 2\lambda(h(\Lambda))
\end{align*}

Now, observe that when $h_{const}=h(\Lambda)$, then $T2=0$, and we need that, solving for
$h_{const}$,
\begin{align*}
    h(\Lambda) > \frac{1}{d(\nu_2,\rho_2)}\log_{\frac{1}{\rho}}\left(\frac{\rho^{-d}-1}{2C(\nu_2,\rho_2)\lambda_{const}}(\Lambda - 2\lambda_{const}) + 1\right)
\end{align*}
and thus, $h(\Lambda)=h_{const}$ scales as $O(\log_{\frac{1}{\rho}}\Lambda)$ w.h.p.

When $\Lambda$ is sufficiently large so that $h(\Lambda) > h_{const}$ and $h_{const}$ 
can be taken as a constant, we need that, for a sufficiently large constant $\hat{C}$,
\begin{align*}
    \hat{C}\log^2\Lambda\frac{\rho^{-(d(\nu_2,\rho_2)+2)h(\Lambda)}-\rho^{-(d(\nu_2,\rho_2)+2)(h_{const}+2)}}{\rho^{-(d(\nu_2,\rho_2)+2)}-1}
    &\overset{want}{>} \Lambda - 2\frac{C\log^2\Lambda}{\nu_1\rho^{2h(\Lambda)}}
\end{align*}
Solving for $h(\Lambda)$, we find that, for some large enough constant $\hat{\hat{C}}$,
we must have that
\begin{align*}
    h(\Lambda) >
    \frac{1}{d(\nu_2,\rho_2)+2}\left(\log_{\frac{1}{\rho}}\frac{\Lambda}{\hat{\hat{C}}\log^2\Lambda}\right)
\end{align*}
and thus, in this case, $h(\Lambda)$ scales as
$O\left(\log_{\frac{1}{\rho}}\frac{\Lambda}{\log^2\Lambda}\right)$ w.h.p.

In the context of Theorem \ref{thm:tree}, this scaling shows that the
simple regret of our algorithm, $R(\Lambda)$, scales roughly
as $\tilde{O}\left(\frac{1}{\Lambda^c}\right)$ for some constant $c$.
Thus, in certain small validation set regimes
as discussed in Remark \ref{rem:smallValidationRegime}, \match{} gives an 
\textit{exponential improvement in simple regret} compared to an 
algorithm which trains only on the validation dataset.
 \section{Additional Experimental Details}\label{sec:expers-appendix}

\subsection{Details about the experimental setup}
 
All experiments were run 
in python:3.7.3 Docker containers (see \url{https://hub.docker.com/_/python}) managed by
\href{https://cloud.google.com/kubernetes-engine/}{Google Kubernetes Engine} running on Google
Cloud Platform on n1-standard-4 instances. Hyperparameter tuning is performed
using the Katib framework (\url{https://github.com/kubeflow/katib}) using the
validation error as the objective.
The code used to create the testing infrastructure can be found at
\url{https://github.com/matthewfaw/mixnmatch-infrastructure},
and the code used to run experiments can be found at
\url{https://github.com/matthewfaw/mixnmatch}.
 
\subsection{Details about the multiclass AUC metric}

We briefly discuss the AUC metric used throughout our experiments. We evaluate
each of our classification tasks using the multi-class generalization of area under the ROC curve (AUROC) proposed by \cite{Hand2001}. This metric considers each pair of classes (i,j),
and for each pair, computes an estimate for the probability that a random sample from class $j$ has lower probability of being labeled as class $i$ than a random sample from class $j$. The metric reported is the average of each of these pairwise estimates.
This AUC genenralization is implemented in the R pROC library
\url{https://rdrr.io/cran/pROC/man/multiclass.html},
and also in the upcoming release of sklearn 0.22.0 
\url{https://github.com/scikit-learn/scikit-learn/pull/12789}.
In our experiments, we use the sklearn implementation.

\subsection{Description of algorithms used}
In the sections that follow, we will reference the following algorithms considered in our experiments. We note that the algorithms discussed
in this section are a superset of those discussed in Section \ref{sec:sims}. 

\begin{table}[h!]
\caption{Description of the algorithms used in the experiments}
\label{tab:algoDescription}
\vskip 0.15in
\centering
\begin{tabular}{ |p{0.25\textwidth}|p{0.7\textwidth}| } 
\hline
Algorithm ID & Description \\
\hline
\texttt{Mix\&MatchCH} & The \match{} algorithm, where the simplex is partitioned using a random coordinate halving scheme \\
\texttt{Mix\&MatchDP} & The \match{} algoirhtm, where the simplex is partitioned using the Delaunay partitioning scheme\\
\texttt{Mix\&MatchCH+0.1Step} & Runs the \texttt{Mix\&MatchCH} algorithm for
the first half of the SGD budget, and runs SGD sampling according to the
mixture returned by \match{} for the second half of the SGD budget, using a
step size 0.1 times the size used by \match{} \\
\texttt{Mix\&MatchDP+0.1Step} & Runs the \texttt{Mix\&MatchDP} algorithm for
the first half of the SGD budget, and runs SGD sampling according to the
mixture returned by \match{} for the second half of the SGD budget, using a
step size 0.1 times the size used by \match{}\\
\hline
\texttt{Genie} & Runs SGD, sampling from the training set according to the test set mixture\\
\texttt{Validation} & Runs SGD, sampling only from the validation set according to the test set mixture\\
\texttt{Uniform} & Runs SGD, sampling uniformly from the training set\\
\texttt{OnlyX} & Runs SGD, sampling only from dataset X\\
\hline
\end{tabular}
\end{table}

\subsection{Allstate Purchase Prediction Challenge -- Correcting for shifted mixtures}
Here, we provide more details about the experiment on the Allstate dataset \cite{kaggle} discussed in Section \ref{sec:sims}.
Recall that in this experiment, we consider the mixture space over which
\match{} searches to be the set of mixtures of 
data from Florida (FL), Connecticut (CT), and Ohio (OH). We take $\aa^*$ to be the proportion of each state in the 
test set.
The breakdown of the training/validation/test
split for the Allstate experiment is shown in Table \ref{tab:allstate}.
\begin{table}[h!]
\caption{The proportions of data from each state used in training, validation, and testing for Figure \ref{fig:allstate3auc} and \ref{fig:allstate_auc_appendix}}
\label{tab:allstate}
\vskip 0.15in
\centering
\begin{tabular}{ |c|c|c|c|c|c| } 
\hline
State & Total Size & \% Train & \% Validate & \% Test & \% Discarded \\
\hline
FL & 14605 & 49.34 & 0.16 & 0.5 & 50\\
CT & 2836 & 50 & 7.5 & 42.5 & 0\\
OH & 6664 & 2.25 & 0.75 & 2.25 & 94.75\\
\hline
\end{tabular}
\end{table}

Here, each \match{} algorithm allocates a height-independent 500 samples for each tree search node on which SGD is run.
Each algorithm uses a batch size of 100 to compute stochastic gradients.

\subsubsection{Dataset transformations performed}

We note that in the dataset provided by Kaggle, the data for a single customer is spread across multiple rows of the dataset, since for each customer there some number (different for various customers) of intermediate transactions, followed by a row corresponding to the insurance plan the customer ultimately selected. We collapse the dataset so that each row corresponds to the information of a distinct customer. To do this, for each customer, we preserve the final insurance plan selected, the penultimate insurance plan selected in their history, the final and penultimate cost of the plan. Additionally, we create a column indicating the total number of days the customer spent before making their final transaction, as well as a column indicating whether or not a day elapsed between intermediate and final purchase, a column indicating whether the cost of the insurance plan changed, and a column containing the price amount the insurance plan changed between the penultimate and final purchase. For every other feature, we preserve only the value in the row corresponding to the purchase. We additionally one-hot encode the car\_value feature. Additionally, we note that we predict only one part of the insurance plan (the G category, which takes 4 possible values). We keep all other parts of the insurance plan as features.

\subsubsection{Experimental results}

Figure \ref{fig:allstate_auc_appendix} shows the results of the same experiment as discussed in Section \ref{sec:sims}. We note
that there are now several variants of the \match{} algorithm, whose implementations are described in Table \ref{tab:algoDescription}.
We observe that, in this experiment, the two simplex partitioning schemes result in algorithms that all have similar performance on the 
test set, and each instance of \match{} outperforms algorithms which train only on a single state's training set as well as the algorithm which
trains only on the validation dataset, which is able to sample using the same mixture as the test set, but with limited amounts of data.
Additionally, each instance of \match{} has either competitive or better performance than the \texttt{Uniform} algorithm, and has performance
competitive with the Genie algorithm.
\begin{figure}[ht]
    \centering
\centering
        \includegraphics[width=.5\linewidth]{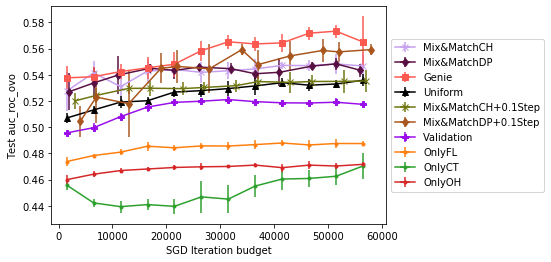}
\caption{Test One vs One AUROC for Mixture of FL, CT, and OH}
\label{fig:allstate_auc_appendix}
\end{figure}

\subsection{Wine Ratings}

We consider the effectiveness of using Algorithm 1 to make predictions on a new region by training on data from other, different regions. For this experiment, we use another Kaggle dataset \cite{wine}, in which we are provided binary labels indicating the presence of particular tasting notes of the wine, as well as a point score of the wine and the price quartile of the wine, for a number of wine-producing countries. 
We will consider several different experiments on this dataset.

We will consider again algorithms discussed in Table \ref{tab:algoDescription}.
Throughout these experiments, we will consider searching over the mixture space of proportions of datasets of wine from countries US, Italy, France, and Spain.
Note that the \texttt{Genie} experiment is
not run since there is no natural choice for $\aa^*,$ as we are aiming to predict on a new country.

\subsubsection{Dataset transformations performed}

The dataset provided through Kaggle consists of binary features describing the country of origin of each wine, as well as tasting notes, and additionally a numerical score for the wine, and the price. We split the dataset based on country of origin (and drop the country during training), and add as an additional target variable the price quartile. We keep all other features in the dataset. In the experiment predicting wine prices, we drop the price quartile column, and in the experiment predicting wine price quartiles, we drop the price column.

\subsubsection{Predict wine prices}

In this section, we consider the task of predicting wine prices in Chile and Australia by using training data from US, Italy, France, and Spain.
The train/validation/test set breakdown is described in Table \ref{tab:wine_npq}.
We use each considered algorithm to train a fully connected neural network with two hidden layers and sigmoid activations, 
similarly as considered in \cite{moew}. We plot the test mean absolute error of each considered algorithm.

Here, each \match{} algorithm allocates a height-independent 500 samples for each tree search node on which SGD is run.
Each algorithm uses a batch size of 25 to compute stochastic gradients.

\begin{table}[h!]
\caption{The proportions of data from each state used in training, validation,
and testing for Figure \ref{fig:wine_npq}}
\label{tab:wine_npq}
\vskip 0.15in
\centering
\begin{tabular}{ |c|c|c|c|c|c| } 
\hline
Country & Total Size & \% Train & \% Validate & \% Test & \% Discarded \\
\hline
US & 54265 & 100 & 0 & 0 & 0 \\
France & 17776 & 100 & 0 & 0 & 0 \\
Italy & 16914 & 100 & 0 & 0 & 0 \\
Spain & 6573 & 100 & 0 & 0 & 0 \\
Chile & 4416 & 0 & 5 & 95 & 0 \\
Australia & 2294 & 0 & 5 & 95 & 0 \\
\hline
\end{tabular}
\end{table}

The results of this experiment are shown in Figure \ref{fig:wine_npq}. There are several interesting takeaways from this experiment.
First is the sensitivity of \match{} to choice of partitioning scheme. While \texttt{Mix\&MatchCH} outperforms the \texttt{Uniform} algorithm
and each \texttt{OnlyX} algorithm, \texttt{Mix\&MatchDP} performs poorly. Note that each node in the search tree under Delaunay partitioning
can have $dim(\mathcal{A})$ ($=4$ in this experiment) children, each node in the coordinate halving scheme only has two children. Thus, it
seems that perhaps the Dealunay partitioning scheme is overly wasteful in its allocation of SGD budget. However, when considering the 
split budget \match{} algorithms which search for mixtures only for half of their SGD budget, and commit to a mixture for the remaining half,
the performance gap between the two partitioning schemes is much less noticeable.

The second interesting takeaway from this experiment is that, in contrast to the other experiments considered in this paper, in this experiment,
it seems that although \match{} outperforms both the \texttt{Uniform} algorithm and \texttt{OnlyX} algorithm, it only matches the performance
of the algorithm which trains only on the validation dataset. This highlights
an important point of the applicability of the \match{} algorithm.
Running \match{} makes sense only when there is insufficient validation data to train a good model.

\begin{figure}[ht]
    \centering
\includegraphics[width=.5\linewidth]{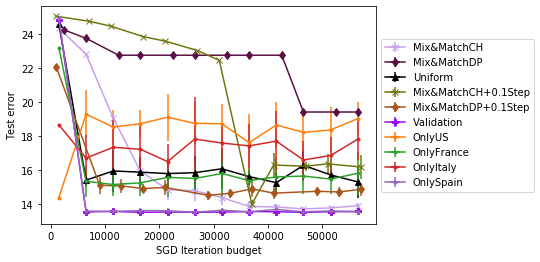}
\caption{Test Mean Absolute Error for Mixture of US, France, Italy, and Spain data, Predict in Chile and Australia}
\label{fig:wine_npq}
\end{figure}

\subsubsection{Predict wine price quartiles}

In this experiment, we consider a classification version of the regression problem considered in the last experiment.
In particular, we have access to training wine data from US, Italy, Spain, and France, and wish to predict the quartile of the wine
price for wines from Chile. The train/validate/test breakdown in given in Table \ref{tab:wine_ppq}.
We use each algorithm to train a fully connected neural network with 3 hidden layers and ReLU activations, and evaluate based on
the One vs One AUROC metric described in \cite{Hand2001}. The experimental results are shown in Figure \ref{fig:wine_ppq}

Here, each \match{} algorithm allocates a height-independent 1000 samples for each tree search node on which SGD is run.
Each algorithm uses a batch size of 25 to compute stochastic gradients.

\begin{table}[h!]
\caption{The proportions of data from each state used in training, validation, and testing for Figure \ref{fig:wine_ppq}}
\label{tab:wine_ppq}
\vskip 0.15in
\centering
\begin{tabular}{ |c|c|c|c|c|c| } 
\hline
Country & Total Size & \% Train & \% Validate & \% Test & \% Discarded \\
\hline
US & 54265 & 100 & 0 & 0 & 0 \\
France & 17776 & 100 & 0 & 0 & 0 \\
Italy & 16914 & 100 & 0 & 0 & 0 \\
Spain & 6573 & 100 & 0 & 0 & 0 \\
Chile & 4416 & 0 & 5 & 95 & 0 \\
\hline
\end{tabular}
\end{table}
\begin{figure}[ht]
    \centering
\centering
        \includegraphics[width=.5\linewidth]{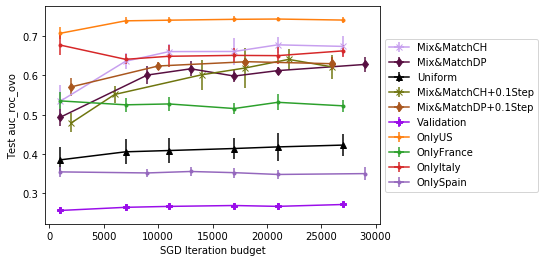}
\caption{Test One vs One AUROC for Mixture of US, France, Italy, and Spain data, Predict in Chile}
\label{fig:wine_ppq}
\end{figure}

We observe that each instance of \match{} outperforms both \texttt{Uniform} and \texttt{Validation} (which has quite poor performance in this experiment), and has competitive performance with the best \texttt{OnlyX} algorithm.

\subsection{Amazon Employee Access Challenge}

Here, we consider yet another Kaggle dataset, the Amazon Employee Access Challenge \cite{amazon}. In this problem, we are given
a dataset of features about employees from a variety of departments, and we wish to predict whether or not a employee is granted access to the system. We note that this dataset is extremely imbalanced, as most requests are approved.

We split this dataset according to department names, with splits shown in Table \ref{tab:amazon}. We use each algorithm to train a fully connected neural network with 3 hidden layers and ReLU activations, and evaluate based on the One vs One AUC metric described in \cite{Hand2001}. 
Here, each \match{} algorithm allocates a height-independent 1000 samples for each tree search node on which SGD is run.
Each algorithm uses a batch size of 50 to compute stochastic gradients.

\subsubsection{Dataset transformations performed}

The dataset provided by Kaggle contains only 10 categorical features.  We one-hot encode each of these features, except ROLE\_DEPTNAME, which we use to split the dataset, and drop during training. Since one-hot encoding these features produces approximately 15000 features, for simplicity, we 
drop from the dataset all features which have fewer than 50 1s.

\subsubsection{Experimental results}

\begin{table}[h!]
\caption{The proportions of data from each state used in training, validation,
and testing for Figure \ref{fig:amazon-appendix}}
\label{tab:amazon}
\vskip 0.15in
\centering
\begin{tabular}{ |c|c|c|c|c|c| } 
\hline
ROLE\_DEPTNAME & Total Size & \% Train & \% Validate & \% Test & \% Discarded \\
\hline
117878 & 1135 & 100 & 0 & 0 & 0 \\
117941 & 763 & 100 & 0 & 0 & 0 \\
117945 & 659 & 100 & 0 & 0 & 0 \\
117920 & 597 & 100 & 0 & 0 & 0 \\
120663 & 335 & 0 & 30 & 70 & 0 \\
\hline
\end{tabular}
\end{table}
\begin{figure}[ht]
    \centering
\centering
        \includegraphics[width=.5\linewidth]{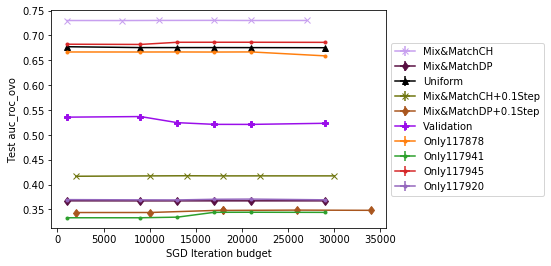}
\caption{Test One vs One AUROC for predicting employee access in a new department}
\label{fig:amazon-appendix}
\end{figure}

The experimental results are shown in Figure \ref{fig:amazon-appendix}.
Here, we observe that \texttt{Mix\&MatchCH} has superior performance to all baseline algorithms. We also observe that,
in this experiment, it seems that spending the entire budget searching over mixtures is more effective than spending only half 
of the SGD budget. Additionally, it seems that using the coordinate halving
partitioning strategy produces better models than \match{} run with the 
Delaunay partitioning scheme.

\end{document}